\documentclass{amsart}
\usepackage{xr}
\makeatletter
\newcommand*{\addFileDependency}[1]{
  \typeout{(#1)}
  \@addtofilelist{#1}
  \IfFileExists{#1}{}{\typeout{No file #1.}}
}
\makeatother

\newcommand{\mtx}[1]{\boldsymbol{#1}}

\DeclareMathOperator*{\argmax}{arg\,max}

\newcommand{\mW}{\mtx{W}}

\DeclareMathOperator{\sign}{sign}

\usepackage{amsthm,amssymb,amsmath,hyperref,xcolor, float, graphicx,tcolorbox, tikz, tikz-cd, algorithm, algpseudocode,overpic, nicematrix}
\usepackage{enumitem}

\hypersetup{
    colorlinks,
    linkcolor={red},
    citecolor={blue},
    urlcolor={blue}
}

\DeclareMathOperator{\Var}{Var}
\DeclareMathOperator{\Tr}{Tr}
\newcommand{\embed}{\boldsymbol{\psi}}
\renewcommand{\epsilon}{\varepsilon}

\newtheorem{lemma}{Lemma}[section]
\newtheorem{theorem}{Theorem}[section]
\newtheorem{corollary}{Corollary}[section]
\newtheorem*{corollary*}{Corollary}
\newtheorem{assumption}{Assumption}[section]

\theoremstyle{remark}
\newtheorem{remark}{Remark}[section]

\title[Laplace-HDC]{Laplace-HDC: Understanding the geometry of binary hyperdimensional computing}

\author[S. Pourmand]{Saeid Pourmand\textsuperscript{\textdagger}}
\email{pourmans@oregonstate.edu}
\address{School of Electrical Engineering and Computer Science, Oregon State University}

\author[W.D. Whiting]{Wyatt D. Whiting\textsuperscript{\textdagger *}}
\email{whitinwy@oregonstate.edu}
\address{Department of Mathematics, Oregon State University}

\author[A. Aghasi]{Alireza Aghasi}
\email{alireza.aghasi@oregonstate.edu}
\address{School of Electrical Engineering and Computer Science, Oregon State University}

\author[N.F. Marshall]{Nicholas F. Marshall}
\email{marsnich@oregonstate.edu}
\address{Department of Mathematics, Oregon State University}

\thanks{\textsuperscript{\textdagger}These authors contributed equally. \textsuperscript{*}Corresponding author}


\begin{document}
\begin{abstract}
This paper studies the geometry of binary hyperdimensional computing (HDC), a computational scheme in which data are encoded using high-dimensional binary vectors. We establish a result about the similarity structure induced by the HDC binding operator and show that the Laplace kernel naturally arises in this setting, motivating our new encoding method \emph{Laplace-HDC}, which improves upon previous methods. We describe how our results indicate limitations of binary HDC in encoding spatial information from images and discuss potential solutions, including using Haar convolutional features and the definition of a translation-equivariant HDC encoding. Several numerical experiments highlighting the improved accuracy of Laplace-HDC in contrast to alternative methods are presented. We also numerically study other aspects of the proposed framework, such as robustness and the underlying translation-equivariant encoding.
\end{abstract}
\maketitle

\section{Introduction}
Hyperdimensional computing (HDC) is a computational paradigm rooted in cognitive science and inspired by the operation of the brain. With billions of neurons and trillions of synapses, the human brain exhibits states akin to high-dimensional arrays. Unlike conventional machine learning models that work with floating-point operations, the brain's processes engage in simpler ``arithmetic'' but across significantly higher dimensions (such as the operations in the cerebral cortex). HDC aims to mimic the brain's operation by encoding data with high-dimensional vectors, called hypervectors, while using simple operations, such as the XOR operation. Hypervectors are often defined randomly or pseudo-randomly and can have entries that are binary, integer, real, or complex \cite{Kleyko2023A}; however,  in this paper, we restrict our attention to binary HDC models, which are the most common in practice.
In contrast with typical floating-point operations on data, the simplicity of binary operations makes binary HDC computationally straightforward and amenable to hardware-level optimization \cite{onlineHD2021}. 

Similar to the cognitive operations of the brain, HDC is robust to noise, heavily distributable, interpretable, and energy efficient \cite{Kanerva2009bg,Thomas2021}. Additionally, HDC models can undergo single-pass training, where a model is trained by processing each sample in the data set only once. The simplicity of arithmetic tasks and their parallel nature facilitate rapid inference for these models. Thanks to these attributes, HDC is a well-suited framework for the Internet of Things (IoT) and edge devices \cite{Khalegi2021,Kleyko2023B}, where resilience to noise and straightforward computations hold significant importance. Despite the simple underlying arithmetic, HDC models are considered across various complex tasks, such as speech recognition \cite{Imani2017}, written language classification \cite{Geethan2021}, DNA pattern matching \cite{kim2020}, robotics \cite{Neubert2019}, image description tasks \cite{neubert2021,schlegel2022comparison}, and low energy computing \cite{basklar2021,Chuang2020,Imani2019,Karunaratne2020}. 

A significant challenge associated with HDC models is their relatively low accuracy. For instance, as demonstrated in the experiments section below, standard HDC models attain an average accuracy rate of 82\% on the MNIST handwritten digit classification task, while several variants of deep neural networks can achieve accuracies above 99.5\% \cite{byerly2021no,hirata2023ensemble,kabir2022spinalnet}. Standard HDC modeling consists of three main steps: hyperdimensional embedding, single-pass training, and inference. Efforts to enhance the accuracy of these models typically involve either refining one of these steps or proposing additional steps in the pipeline. For example, a standard HDC model uses a record-based encoding, which involves the binding and bundling of random hypervectors associated with the feature positions and feature values, while to boost the accuracy for temporal or spatial data, an $N$-gram-based encoding may be considered, where feature positions are encoded into the hypervectors through rotational permutations \cite{Ge2020Classification}. As an example of changing the number of steps in the process, OnlineHD \cite{onlineHD2021} is an HDC framework where an additional retraining phase is integrated into the modeling pipeline. This augmentation boosts the model accuracy by discarding the mispredicted queries from the corresponding mispredicted classes and adding them to the correct class.

\subsection{Motivation} \label{sec:motivation}
In this paper, we are primarily motivated by the work of Yu \emph{et al.} \cite{yu2022understanding}, which, in contrast to previous works that mainly focused on using either deterministic hypervectors or random hypervectors with i.i.d. entries, considers
constructions of random hypervectors with a prescribed covariance structure. They refer to their approach as RFF-HDC and empirically show that it outperforms previous HDC schemes.
More precisely, given a desired covariance structure $\boldsymbol{K} \in \mathbb{R}^{m \times m}$, Yu \emph{et al.} \cite{yu2022understanding}
uses the following algorithm  to construct a matrix $\boldsymbol{V} \in \{-1,+1\}^{N \times m}$ of $m$ hypervectors.

\begin{minipage}{\textwidth}
 \begin{algorithmic}
        \Require Similarity matrix $\boldsymbol{K} \in \mathbb{R}^{m \times m}$, hyperdimension $N$
        \State $\boldsymbol{W} \gets \sin\left(\frac{\pi}{2} \boldsymbol{K} \right)$ \qquad (where the function $\sin$ is applied entrywise)
        \State $\boldsymbol{U} \boldsymbol{S} \boldsymbol{U}^T \gets \boldsymbol{W}$ \qquad (eigendecomposition)
        \State Generate $\boldsymbol{G} \in \mathbb{R}^{N \times m}$ with i.i.d. standard Gaussian entries
        \State $\boldsymbol{V} \gets
        \sign
        (\boldsymbol{G} \boldsymbol{S}_+^{1/2} \boldsymbol{U})$
        \qquad (where $\sign$ is applied entrywise and $\boldsymbol{S}_+$ sets negative entries to $0$).
        \State \Return $\boldsymbol{V} \in \{-1,+1\}^{N \times m}$
\end{algorithmic}
\end{minipage}

When $\boldsymbol{S} = \boldsymbol{S}_+$, that is, when $\boldsymbol{W}$ is positive semi-definite, it follows that $\boldsymbol{V}^\top \boldsymbol{V}/N = \boldsymbol{K}$, see \S \ref{definehypervec} for a more precise statement and mathematical description. 
The assumption that $\boldsymbol W$ is positive semi-definite constrains the types of similarities that can be achieved using this algorithm. However, Yu \emph{et al.} \cite{yu2022understanding} shows that some restriction is necessary by proving that there are covariance structures $\boldsymbol{K}$ which are impossible to realize using binary hypervectors.

The current paper builds upon Yu \emph{et al. } \cite{yu2022understanding} by studying the geometry of HDC encodings resulting from constructions using hypervectors with a covariance structure. More precisely, we consider the similarity structure of the embedding space induced by the HDC binding operation, see \eqref{generalencoding:scheme}. We show that the Laplace kernel naturally arises in this context, and our results provide heuristics for choosing effective distributions of hypervectors. Moreover, we consider several modifications to the HDC pipeline, which further boost the accuracy of HDC models. We demonstrate theoretically and empirically how spatial information for images is lost in the similarity structure of certain HDC encoding schemes and present methods of retaining this information, including the definition of a translation-equivariant HDC encoding scheme. In addition to conducting empirical experiments to assess the proposed framework's performance compared to state-of-the-art techniques, mathematical tools are used to explore theoretical aspects. 

We emphasize that some models we explore (as with previous work) involve using floating-point operations during the construction of the hypervectors or training stages of the models. We will emphasize when this is the case. Moreover, as we will discuss, these models have variants that can ultimately be fully represented and operated in a binary mode for inference, and results for binary inference will be presented.

\subsection{Preliminaries and Notation}
While implementations of binary HDC use vectors $\boldsymbol{x} \in \{0,1\}^N$ (for some large $N$ on the order of $N = 10^4$) equipped with entrywise XOR (denoted $\oplus$),
we may conceptualize these vectors as being in $\{-1,+1\}^N$ equipped with entrywise product (denoted $\odot$). These two representations are isomorphic: if $\phi: \{0,1\}^N\rightarrow \{-1,+1\}^N$ by 
$\boldsymbol{x} \mapsto \phi(\boldsymbol{x}) = \boldsymbol{1} - 2 \boldsymbol{x}$ then
$$
\phi(\boldsymbol x(i)\oplus \boldsymbol y(i)) = 1 - 2(\boldsymbol x(i)\oplus \boldsymbol y(i)) = \phi(\boldsymbol x(i)) \odot \phi(\boldsymbol y(i)),
$$
for all $i = 1,\ldots,N$, where $\boldsymbol x(i)$ denotes the $i$-th entry of $\boldsymbol{x}$. In this paper, we use the $(\{-1,+1\}^N, \odot)$ representation of binary HDC schemes.

For a given hypervector $\boldsymbol{u}   \in \{-1,+1\}^N$, we denote its $k$-th entry 
by $\boldsymbol{u}(k)$. We write $\boldsymbol{e}_k$ to denote the $k$-th standard basis vector whose $k$-th entry is equal to $1$ and which is zero elsewhere. We use the convention that hypervectors are column vectors such that the inner product can be expressed by 
$$
\boldsymbol{u}^\top \boldsymbol{v}  = \sum_{k=1}^N \boldsymbol{u}(k) \boldsymbol{v}(k),
$$
where $\boldsymbol{u}^\top$ denotes the transpose of $\boldsymbol{u}$, and write $\boldsymbol{u} \odot \boldsymbol{v}$ to denote the entrywise product 
$$
(\boldsymbol{u} \odot \boldsymbol{v})(k) = \boldsymbol{u}(k) \boldsymbol{v}(k),
$$
for $k = 1,\ldots, N$. Given a set of $d$ hypervectors $\boldsymbol{v}_1,\ldots,\boldsymbol{v}_d$, let
$$
\bigodot_{j=1}^d \boldsymbol{v}_j = \boldsymbol{v}_1 \odot \cdots \odot \boldsymbol{v}_d,
$$
denote the entrywise product over all vectors in the set.  For a matrix $\boldsymbol{A} \in \mathbb{R}^{N \times N}$, we write $\|\boldsymbol{A} \|_0$ to denote the number of nonzero entries of $\boldsymbol{A}$
$$
\|\boldsymbol{A}\|_0 := \# \{ (i,j) \in \{1,\ldots,N\}^2 : \boldsymbol{A}(i,j) \not = 0 \},
$$
where $\boldsymbol{A}(i,j)$ denotes the $(i,j)$-th entry of $\boldsymbol{A}$, and $\#$ denotes the counting measure.

\subsection{Defining an Embedding} 
Assume that data $\mathcal{X} \subset \{1,\ldots,m\}^d$ are given, that is, each $\boldsymbol{x} \in \mathcal{X}$ is a $d$-dimensional vector whose entries are integers in the set $\{1,\ldots,m\}$. Let $m$ hypervectors
$$
\boldsymbol{v}_1,\ldots,\boldsymbol{v}_{m} \in \{-1,+1\}^N,
$$
be given (see \S \ref{definehypervec} for a discussion of constructing these hypervectors).
Further, let $\mathcal{P}$ be a set $d$ permutation matrices of size $N\times N$
$$
\mathcal{P}= \{\boldsymbol{\Pi}_1, \boldsymbol{\Pi}_2, \ldots, \boldsymbol{\Pi}_d\};
$$
 examples of families of permutation matrices of interest are discussed below.
The binding operation which maps $\mathcal{X} \rightarrow \{-1,+1\}^N$ is defined by
\begin{equation}\label{generalencoding:scheme}
\boldsymbol{x} \mapsto \embed_{\boldsymbol{x}}  = \bigodot_{i=1}^d \boldsymbol{\Pi}_i \boldsymbol{v}_{\boldsymbol{x}(i)}.
\end{equation}
For the encoding $\boldsymbol{\psi}_{\boldsymbol{x}}$ to be meaningful, some assumptions must be imposed on the permutation matrices. We make the following trace-orthogonality assumption.
\begin{assumption}[Trace-orthogonal family of permutations] \label{assumenooverlap} 
We say that a family of permutations 
$\mathcal{P}= \{\boldsymbol{\Pi}_1, \boldsymbol{\Pi}_2, \ldots, \boldsymbol{\Pi}_d\}$ is
trace-orthogonal if
\begin{equation} \label{nooverlap}
\left\langle \boldsymbol{\Pi}_i, \boldsymbol{\Pi}_{i'}\right\rangle = \Tr\left(\boldsymbol{\Pi}_i^\top \boldsymbol{\Pi}_{i'}\right) =  0, \qquad \forall i,i'\in\{1,\ldots,d\}, ~i\neq i'.
\end{equation}
\end{assumption}
It is straightforward to construct a family of trace-orthogonal permutations using cyclic shifts (under the necessary assumption that $d \leq N$).

\begin{remark}[1D-Cyclic family] \label{cyclicshifts}
For $i \in \{1,\ldots,d\}$ let $\boldsymbol{T}_i^\text{1D-Cyclic}$ denote the $N \times N$ permutation matrix, which acts on $\boldsymbol{v} \in \{-1,+1\}^N$ by
\begin{equation} \label{cyclicshiftseq}
(\boldsymbol{T}_i^\text{1D-Cyclic} \boldsymbol{v})(i') = \boldsymbol{v}(i+i'), \quad \text{for} \quad i' \in \{1,\ldots,N\},
\end{equation}
where the addition $i+i'$ is taken modulo $N$. That is, 
$\boldsymbol{T}_i^\text{1D-Cyclic}$ can be defined entrywise by
$$
\boldsymbol{T}_i^\text{1D-Cyclic}(j,j') = \left\{
\begin{array}{cl}
1 & \text{if } j = j' + i \mod N \\
0 & \text{otherwise}.
\end{array}
\right.
$$
When $i \not = i'$ and $d \le N$, the support of $\boldsymbol{T}_i^\text{1D-Cyclic}$ and $\boldsymbol{T}_{i'}^\text{1D-Cyclic}$ are disjoint so the trace-orthogonal property holds.
\end{remark}

\begin{remark}[1D-Block Cyclic family] \label{1dblock}
Suppose that $N = d M$ for some positive integer $M$. Then, another family of trace-orthogonal permutations matrices  $\{ \boldsymbol{T}_i^\text{1D-Block} : i \in \{1,\ldots,d\}\}$ can be defined by their action on $\boldsymbol{v} \in \{-1,+1\}^{d \times M}$ by
\begin{equation} \label{1dblockeq}
(\boldsymbol{T}_i^\text{1D-Block} \boldsymbol{v})(i',k) = \boldsymbol{v}(i+i',k), \quad \text{for} \quad (i',k) \in \{1,\ldots,d\} \times \{1,\ldots,M\},
\end{equation}
where the addition $i+i'$ is taken modulo $d$. It is straightforward to verify that this 1D-Block Cyclic family is also trace-orthogonal.
\end{remark}

\subsection{Constructing Hypervectors} \label{definehypervec} 
We choose hypervectors using the method of
Yu \emph{et al.} \cite{yu2022understanding} outlined in \S \ref{sec:motivation} above. In the following, 
we describe this construction in detail and provide related mathematical preliminaries.
Given an affinity matrix $\boldsymbol{K} \in \mathbb{R}^{m \times m}$, the goal is to construct hypervectors $\boldsymbol{v}_1,\ldots,\boldsymbol{v}_m \in \{-1,+1\}^N$ such that
$$
\mathbb{E} \frac{\boldsymbol{v}_i^\top \boldsymbol{v}_j}{N} =  \boldsymbol{K}(i,j).
$$
Recall Grothendieck's identity (see, for example, Vershynin \cite[page 63]{vershynin2018high}).

\begin{lemma}[Grothendieck's identity]  \label{grothendieck}
Let $\boldsymbol{g}$ be an $n$-dimensional vector with i.i.d. random standard Gaussian entries.
Then, for any fixed vectors $\boldsymbol{u},\boldsymbol{v} \in \mathbb{S}^{n-1}$, we have
$$
\mathbb{E}( \sign(\boldsymbol{g}^\top \boldsymbol{u}) \sign(\boldsymbol{g}^\top \boldsymbol{v}) ) = \frac{2}{\pi} \arcsin( \boldsymbol{u}^\top \boldsymbol{v} ),
$$
where $\mathbb{S}^{n-1} = \{ x \in \mathbb{R}^n : \|x\|_2 = 1\}$.
\end{lemma}

Suppose that an affinity kernel matrix $\boldsymbol{K} \in \mathbb{R}^{m \times m}$ is given, and define $\boldsymbol{W} \in \mathbb{R}^{m \times m}$ by
\begin{equation} \label{eqWdef}
\boldsymbol{W}(i,j) = \sin \left( \frac{\pi}{2} \boldsymbol{K}(i,j) \right),
\end{equation}
for $i,j = 1,\ldots,m$. The construction is effective when the following assumption, which restricts the possible choices of $\boldsymbol{K}$, is satisfied.
\begin{assumption}[Admissible affinity kernel] \label{admissiblekernel} We say that an affinity kernel $\boldsymbol{K} \in \mathbb{R}^{m \times m}$ is admissible if the matrix $\boldsymbol{W}$ defined by \eqref{eqWdef} is symmetric positive semi-definite and $\boldsymbol{K}(i,i) = 1$ for all $i = 1,\ldots,m$.
\end{assumption}
An example of a family of admissible affinity kernels is provided in Corollary \ref{cor1}; also see \S \ref{admissiblekernel}. For now, 
we proceed under the assumption that $\boldsymbol{W}$ is a symmetric positive semi-definite matrix, which implies $\boldsymbol{W}$ can be decomposed as 
$$
\boldsymbol{W} = \boldsymbol{U}^\top  \boldsymbol{U},
$$
where $\boldsymbol{U}$ is a real-valued $n \times n$ matrix. 
Note that if $\boldsymbol{W}$ is not positive semi-definite, it is possible to truncate the negative eigenvalues to achieve a decomposition of this form, see \S \ref{sec:motivation}; however, in this case, the construction will not achieve hypervectors with the covariance structure of $\boldsymbol{K}$.

Let $\boldsymbol{G} \in \mathbb{R}^{N \times m}$ be a matrix whose entries are independent Gaussian random variables with mean $0$ and variance $1$. Define the matrix $\boldsymbol{V} \in \{-1,+1\}^{N \times m}$ by
$$
\boldsymbol{V} = \sign( \boldsymbol{G} \boldsymbol{U} ),
$$
where $\sign$ denotes the entrywise sign function 
with the convention that $0$ has sign $+1$. Let $\boldsymbol{v}_k$ denote the $k$-th column of $\boldsymbol{V}$. Then, the hypervectors
$$
\boldsymbol{v}_1,\ldots,\boldsymbol{v}_m \in \{-1,+1\}^N,
$$
have the desired covariance structure in expectation. Indeed,
we claim that
\begin{equation}\label{simvivj:scalar}
    \mathbb{E} \boldsymbol{v}_i(k) \boldsymbol{v}_j(k) =  \boldsymbol{K}(i,j),\quad \text{for all} \quad k\in \{1,\ldots,N\},
\end{equation}
which in turn implies
\begin{equation} \label{simvivj}
\mathbb{E} \frac{\boldsymbol{v}_i^\top \boldsymbol{v}_j}{N} =  \boldsymbol{K}(i,j).
\end{equation}
To show \eqref{simvivj:scalar}, we note that by the definition of $\boldsymbol{v}_j$ we have
$$
\mathbb{E} \boldsymbol{v}_i(k) \boldsymbol{v}_j(k) = \mathbb{E} \sign(\boldsymbol{g}_k^\top \boldsymbol{u}_i) \sign(\boldsymbol{g}_k^\top \boldsymbol{u}_j),
$$
where $\boldsymbol{g}_k^\top$ is the $k$-th row of $\boldsymbol{G}$ and $\boldsymbol{u}_j$ is the $j$-th column of $\boldsymbol{U}$. Recall that in the definition of an admissible kernel, we assume
 $\boldsymbol{K}(i,i) = 1$. It follows that 
$$
\|\boldsymbol{u}_i \|_2^2 = \boldsymbol{u}_i^\top \boldsymbol{u}_i = \boldsymbol{W}(i,i) = \sin\left(\frac{\pi}{2} \boldsymbol{K}(i,i) \right) = 1,
$$ 
for $i=1,\ldots,m$. Thus, Grothendieck's Identity  (see Lemma \ref{grothendieck} above) can be applied to deduce that
\begin{equation} \label{expect}
 \mathbb{E} \sign(\boldsymbol{g}_k^\top \boldsymbol{u}_i) \sign(\boldsymbol{g}_k^\top \boldsymbol{u}_j)
=  \frac{2}{\pi} \arcsin( \boldsymbol{u}_i^\top \boldsymbol{u}_j ).
\end{equation}
By the definition of $\boldsymbol{U}$ we have
$$
\frac{2}{\pi} \arcsin( \boldsymbol{u}_i^\top \boldsymbol{u}_j ) = \frac{2}{\pi} \arcsin \left( \boldsymbol{W}(i,j) \right) = \boldsymbol{K}(i,j),
$$
where the final equality follows from the definition of $\boldsymbol{W}$.

\subsection{Main Analytic Result} \label{mainresult}

Let data $\mathcal{X} \subset \{1,\ldots,m\}^d$ be given, and  fix a hyperdimension $N \ge d$. Let $\mathcal{P}= \{\boldsymbol{\Pi}_1, \boldsymbol{\Pi}_2, \ldots, \boldsymbol{\Pi}_d\}$ be a family of $N \times N$ permutation matrices satisfying Assumption \ref{assumenooverlap}. Assume that $\boldsymbol{K} \in \mathbb{R}^{m \times m}$ is an affinity kernel which is admissible in the sense of Assumption \ref{admissiblekernel}. Construct $\boldsymbol{v}_1,\ldots,\boldsymbol{v}_{m} \in \{-1,+1\}^N$
using $\mathcal{P}$ and $\boldsymbol{K}$ in the procedure described in \S \ref{definehypervec}. Define the embedding $\embed_{\boldsymbol{x}} \in \{-1,+1\}^N$ of a vector $\boldsymbol{x} \in \mathcal{X}$ in terms of the binding operation
\begin{equation}\label{generalencoding:scheme2}
\boldsymbol{x} \mapsto \embed_{\boldsymbol{x}}  = \bigodot_{i=1}^d \boldsymbol{\Pi}_i \boldsymbol{v}_{\boldsymbol{x}(i)}.
\end{equation}
The following theorem is our main analytic result.

\begin{theorem} \label{thm1} Under the assumptions of \S \ref{mainresult}, we have
\begin{equation} \label{thm1sxy}
S(\boldsymbol{x},\boldsymbol{y}) := \mathbb{E} \frac{\embed_{\boldsymbol{x}}^\top \embed_{\boldsymbol{y}}}{N}  = \prod_{i=1}^d  \boldsymbol{K}(\boldsymbol{x}(i),\boldsymbol{y}(i)),  
\end{equation}
and
$$
\Var \left( \frac{ \embed_{\boldsymbol{x}}^\top \embed_{\boldsymbol{y}}}{N} \right)   \le \frac{2\gamma_\mathcal{P}}{N^2} (1-S(\boldsymbol{x},\boldsymbol{y})),
$$
where for the set $\mathcal{P}= \{\boldsymbol{\Pi}_1, \boldsymbol{\Pi}_2, \ldots, \boldsymbol{\Pi}_d\}$:
\begin{equation} \label{gammadef}
\gamma_{\mathcal{P}} = \left\|\sum_{i'=1}^d\sum_{i=1}^d\boldsymbol{\Pi}_i\boldsymbol{\Pi}_{i'}^\top \right\|_0. 
\end{equation}
\end{theorem}

The proof of Theorem \ref{thm1} is given in \S \ref{proofthm1}.
The following corollary states this result for the case where $\mathcal{P}$ is the 1D-Cyclic family of permutation matrices, and $\boldsymbol{K}$ is approximately the Laplace kernel. To simplify the exposition, we present an informal version of this result here and provide a more technical statement in the following section.

\begin{corollary}[Informal Statement] \label{cor1} 
Let $\mathcal{P} = \{\boldsymbol{\Pi}_1,\ldots,\boldsymbol{\Pi}_d\}$ be the 1D-Cyclic family of permutation matrices defined in Remark \ref{cyclicshifts}, which satisfies Assumption \ref{assumenooverlap}.
It is possible to choose $\boldsymbol{K}$ approximately equal to the Laplace kernel $\boldsymbol{K}(i,j) \approx \exp({-\lambda|i - j|})$ such that Assumption \ref{admissiblekernel} is satisfied. In this case,
\begin{equation} \label{cor1sxy}
S(\boldsymbol{x}, \boldsymbol{y}) :=   \mathbb{E} \left( \frac{\boldsymbol{\psi}_{\boldsymbol{x}}^\top \boldsymbol{\psi}_{\boldsymbol{y}}}{N} \right) \approx \exp(- \lambda \|\boldsymbol{x} - \boldsymbol{y}\|_1).
\end{equation}
Moreover, the variance satisfies
$$
\Var \left( \frac{ \embed_{\boldsymbol{x}}^\top \embed_{\boldsymbol{y}}}{N} \right)   \le \frac{4d-2}{N} (1- S(\boldsymbol{x},\boldsymbol{y})).
$$
\end{corollary}

For a precise version of the corollary, see  Theorem \ref{admissiblefamily} below. 
Interestingly, it is not possible to replace the Laplace kernel with the Gaussian kernel in this statement, see \S 
\ref{sec:admissiblekernel}

\subsection{Choosing an Admissible Kernel} \label{sec:admissiblekernel}

Recall that $\boldsymbol{K} \in \mathbb{R}^{m \times m}$ is admissible in the sense of Assumption \ref{admissiblekernel} if $\boldsymbol{K}(i,i) = 1$ and if $\boldsymbol{W}$ defined by
$$
\boldsymbol{W}(i,j) = \sin \left( \frac{\pi}{2} \boldsymbol{K}(i,j) \right)
$$
is positive semi-definite. In \S \ref{informalderivation}, we conduct an informal study of admissible kernels. Subsequently, in \S \ref{formalstatement}, we state a precise result that describes a family of admissible kernels $\boldsymbol{K}_\alpha$ and the resulting expected similarity ${\boldsymbol{S}}_\alpha(\boldsymbol{x},\boldsymbol{y})$.

\subsubsection{Heuristic Derivation} \label{informalderivation}
Given an admissible kernel $\boldsymbol{K}$, Theorem \ref{thm1} states that the resulting expected similarity $S(\boldsymbol{x},\boldsymbol{y})$ has the form 
$$
S(\boldsymbol{x},\boldsymbol{y}) = \prod_{i=1}^d  \boldsymbol{K}(\boldsymbol{x}(i),\boldsymbol{y}(i)).
$$
This formula suggests that the values of $\boldsymbol{K}(i,j)$ should all be relatively close to $1$ for all $i,j$ since, otherwise, taking the product will result in values close to zero; this observation motivates the ansatz   
$$
\boldsymbol{K}(i,j) =1 - \boldsymbol{F}(i,j), \quad
\text{where} \quad 0 \le \boldsymbol{F}(i,j) \le \varepsilon,
$$
for some small $\varepsilon > 0$. The series expansions
$$
\exp(x) = 1 + x + \mathcal{O}(x^2), \quad \text{and} \quad
\sin\left( \frac{\pi}{2}(1 - x) \right) = 1 - \frac{\pi^2 x^2}{8} + \mathcal{O}(x^4), 
$$
as $x \rightarrow 0$  imply that the resulting matrix $\boldsymbol{W}$ satisfies
\begin{equation*}
\boldsymbol{W}(i,j) = \sin \left( \frac{\pi}{2} \left( 1 -\boldsymbol{F}(i,j)  \right) \right) 
= \exp\left( - \frac{\pi^2}{8}  \boldsymbol{F}(i,j)^2 \right) + \mathcal{O}(\varepsilon^4) .
\end{equation*}
These calculations motivate choosing $\boldsymbol{F}(i,j)$ such that
$$
\boldsymbol{W}(i,j) \approx \exp\left( - \frac{\pi^2}{8}  \boldsymbol{F}(i,j)^2 \right)
$$
is positive semi-definite. One natural choice is setting $ \boldsymbol{F}(i,j) = \lambda |i - j|$ such that 
$\boldsymbol{K}$ is approximately equal to the Laplace kernel $\boldsymbol{K}(i,j) \approx \exp(-\lambda |i-j|)$
and $\boldsymbol{W}$ is approximately equal to the Gaussian kernel
$$
\boldsymbol{W}(i,j) \approx \exp \left( - |i -j|^2/(2\sigma^2) \right) ,
$$
where $\sigma = 2\lambda/\pi$.  Both the Laplace kernel and the Gaussian kernel are positive definite kernels, so the Laplace kernel $\boldsymbol{K}$ is one natural choice to use in this construction. 

Interestingly, choosing $ \boldsymbol{F}(i,j) = |i - j|^2/(2\sigma^2)$ such that 
$\boldsymbol{K}$ is approximately equal to the Gaussian kernel $\boldsymbol{K}(i,j) \approx \exp(-\lambda |i-j|^2/(2\sigma^2)$ results in $\boldsymbol{W}$ of the form
$$
\boldsymbol{W}(i,j) \approx \exp \left( - \gamma |i -j|^4 \right),
$$
for $\gamma = \pi^2/(32 \sigma^4)$,  which is not a positive definite kernel. Thus, the Gaussian kernel is not an admissible choice of $\boldsymbol{K}$ for this construction.

\subsubsection{Precise Description} \label{formalstatement}

In the following, we make the informal derivation of the previous section, which says the Laplace kernel is a natural choice for $\boldsymbol{K}$ more rigorous and general. In particular, we define a family of admissible kernels and derive the resulting expected similarity kernel. 

\begin{lemma} \label{lemschoenberg}
    Fix $\alpha, \lambda > 0$, and consider the matrix $\boldsymbol{W}_\alpha\in\mathbb{R}^{m\times m}$ with elements 
    \begin{equation}
       \boldsymbol{W}_\alpha(i,j) = \exp\left( -\lambda |i-j|^\alpha\right),\label{eq:expkernelclass} 
    \end{equation}
    for all $i,j = 1,\ldots,m$. Then,
    \begin{itemize}
    \item[--] For each $\alpha\in[0,2]$, $\mW_\alpha$ is  positive semi-definite 
    \item[--] For each $\alpha\in(2,\infty)$,  $\mW_\alpha$ is not positive semi-definite 
    \end{itemize}
\end{lemma}
\begin{proof}[Proof of Lemma \ref{lemschoenberg}]
    This is a direct implication of Schoenberg's classic result. Specifically, Corollary 3 of Schoenberg \cite{schoenberg1938metric} states that $\exp\left(-|x|^\alpha\right)$ is positive definite when $0<\alpha\leq 2$, and fails to become positive definite when $\alpha>2$. This function being positive definite implies that for any selection of $x_1,\ldots,x_m\in\mathbb{R}$, the matrix with elements $\exp\left(-|x_i-x_j|^\alpha\right)$ is positive semi-definite. Narrowing down the choice of the test points to $x_i=\lambda^{1/\alpha}i$ for $i =1,\ldots,m$ validates the claims of the lemma. 
\end{proof}

The following theorem defines a family of admissible kernels and, motivated by the informal derivation of \S \ref{informalderivation}, derives the resulting expected similarity kernel. A subsequent remark describes the connection to the Laplace kernel derived in the previous section.

\begin{theorem} \label{admissiblefamily} Let $\lambda > 0$ and $\alpha \in (0,1]$. Define
$\boldsymbol{K}_\alpha \in \mathbb{R}^{m \times m}$ by
\begin{equation} \label{admissiblekerneleq}
\boldsymbol{K}_\alpha(i,j) = \frac{2}{\pi} \arcsin \left( \exp \left( - \frac{\pi^2}{8} \lambda |i - j|^{2 \alpha} \right) \right),
\end{equation}
for $i,j \in \{1,\ldots,m\}$. Then, $\boldsymbol{K}_\alpha$ is admissible in the sense of Assumption \ref{admissiblekernel}. Moreover, if the bandwidth parameter $\lambda$ is set such that  $\lambda |i -j|^\alpha \le \varepsilon$, then 
$$
\boldsymbol{K}_\alpha(i,j) = \exp \left( - \lambda |i -j|^\alpha \right) + \mathcal{O}(\varepsilon^2),
$$
and the resulting expected similarity kernel $S_\alpha$ satisfies
$$
S_\alpha(\boldsymbol{x},\boldsymbol{y})  =  \exp( -\lambda \|\boldsymbol{x} - \boldsymbol{y}\|_\alpha^\alpha )(1 + \mathcal{O}(\varepsilon^2 d)),
$$
as $\varepsilon \rightarrow 0$.
\end{theorem}

The proof of Theorem \ref{admissiblefamily} is given in \S \ref{proofthm2}. In the following, we make three remarks about this result.

\begin{remark}[Laplace kernel]
When $\alpha =1$
$$
\boldsymbol{K}_\alpha(i,j) = \exp(-\lambda|i-j|) + \mathcal{O}(\varepsilon^2),
$$
is the Laplace kernel, up to lower order terms, and 
$$
\boldsymbol{S}_\alpha(i,j) = \exp(-\lambda\|\boldsymbol{x} - \boldsymbol{y}\|_1)(1 + \mathcal{O}(\varepsilon^2 d)),
$$
which recovers the informal motivating result from the previous section.
\end{remark}

\begin{remark}[Limitation of similarity structure] \label{limitation}
The result of Theorem \ref{thm1}  suggests a limitation of binary HDC pipelines that involve binding data $\boldsymbol{x} \mapsto \boldsymbol{\psi}_{\boldsymbol{x}}$ and then relying on the similarity structure induced by inner products $\boldsymbol{\psi}_{\boldsymbol{x}}^\top \boldsymbol{\psi}_{\boldsymbol{y}}/N$. Namely, the expected similarity $S(\boldsymbol{x},\boldsymbol{y})$ is invariant to global permutations of the elements of the data, that is,
$$
S(\boldsymbol{x},\boldsymbol{y}) = S(\boldsymbol{x} \circ \sigma,\boldsymbol{y} \circ \sigma),
$$
where $(\boldsymbol{x} \circ \sigma)(i) = \boldsymbol{x}(\sigma(i))$ for a permutation $\sigma$, see \eqref{thm1sxy}. For some data types, such as images, the ordering of the data elements may contain information. For example, a group of black pixels may indicate structure, such as a digit, while scattered black pixels may be noise. These spatial relationships can be captured using feature extraction methods such as convolutional filters, see \S \ref{sec:convfeatures}. Alternatively, the definition of the embedding could be modified so that spatial information is encoded via another mechanism such as translation-equivariance, see \S \ref{sec:translationequivar}
\end{remark}

\begin{remark}[Setting the bandwidth parameter $\lambda$] \label{rem:settingbandwidth}

When $\boldsymbol{K}_\alpha$ is defined by \eqref{admissiblekerneleq},
the form of the resulting expected similarity 
$$
S_\alpha(\boldsymbol{x},\boldsymbol{y})  \approx \exp( -\lambda \|\boldsymbol{x} - \boldsymbol{y}\|_\alpha^\alpha ),
$$
can be used to set the bandwidth parameter $\lambda$. Suppose a data set $\mathcal{X} = \{ \boldsymbol{x}_i \}_{i=1}^n$ is given. A typical strategy for kernel methods is setting the bandwidth parameter so that $\boldsymbol{K}_\alpha$ is invariant to transformations that preserve distances $\|\boldsymbol{x} - \boldsymbol{y}\|_\alpha$ up to a global constant (for example, global scaling of the data). One way to achieve this is by setting
\begin{equation} \label{bandwidthset}
\lambda = \frac{c}{ \text{median}(\boldsymbol{D})},
\end{equation}
for some constant $c > 0$, where
$$
\boldsymbol{D}(i,j) = \|\boldsymbol{x}_i - \boldsymbol{x}_j\|_\alpha^\alpha,
$$
and $\text{median}(\boldsymbol{D})$ is the median of the $n^2$ numbers in the matrix $\boldsymbol{D}$. 
When the number of data points $n$ is large, the matrix $\boldsymbol{D}$ in  \eqref{bandwidthset} could be replaced by a matrix $\boldsymbol{D}'$ corresponding to a subset of the data to reduce the required computation.
\end{remark}

\subsection{Translation Equivariant Encoding}
\label{sec:translationequivar}
Remark \ref{limitation} describes a limitation of the method related to encoding spatial information.
In the following, we describe an alternate way to encode spatial information by defining a translation equivariant binding operation for images. Suppose that $\boldsymbol{x} \in \{1,\ldots,m\}^{L \times L}$ is an $L \times L$ image. For simplicity, initially assume $N:= L^2$ (We will subsequently relax this assumption to handle $N > L^2$).
 We define a family of permutations
$
\mathcal{P} = \left\{\boldsymbol{T}_{i, j} : i, j = 1, ...,\, m \right\},
$
which act on $\boldsymbol v \in \{-1,+1\}^{L \times L}$ by
$$
\left( \boldsymbol{T}_{i, j}\boldsymbol v \right)(i',j') = \boldsymbol v (i+ i', j+j'),
$$
where the addition $i+i'$ and $j+j'$ is taken modulo $L$. This family satisfies the Trace-Orthogonality Assumption \ref{assumenooverlap}.  Define the embedding
\begin{equation}\label{2dencoding:scheme}
\boldsymbol x\mapsto \boldsymbol{\psi}_{\boldsymbol{x}} = \bigodot_{i,j=1}^{L} \boldsymbol{T}_{i, j} \boldsymbol v_{\boldsymbol x(i, j)},
\end{equation}
where $\boldsymbol x(i, j)$ denotes the $(i, j)$-th entry of $\boldsymbol{x}$. By construction, this embedding is translation-equivariant. That is, 
$$
\boldsymbol T_{i,j} \boldsymbol{\psi}_{\boldsymbol{x}} = \boldsymbol{\psi}_{\boldsymbol T_{i,j} \boldsymbol x},
$$
see Figure \ref{fig-equi}. 
\begin{figure}[ht!]
\centering
\begin{tikzcd}[column sep=1cm, row sep=0.5cm, every label/.append style={font=\large}]
\text{\textbf{Image space}} & \text{\textbf{Embedding space}} \\[-0.5cm]
\raisebox{-.5\height}{\includegraphics[width=.1\textwidth]{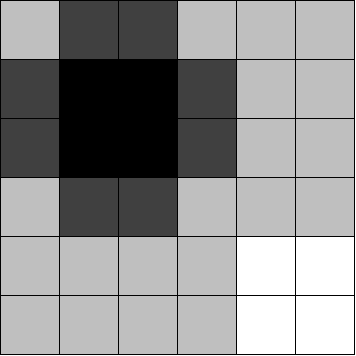}}
 \arrow[r, "\embed"] \arrow[d, "\boldsymbol{T}_{i,j}"] 
 &
 {\raisebox{-.5\height}{\includegraphics[width=.1\textwidth]{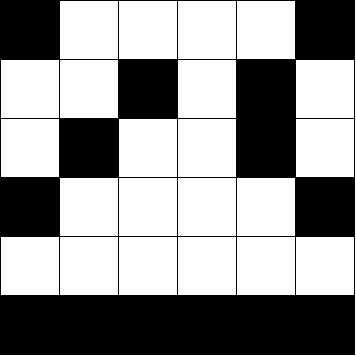}}}  
\arrow[d, "\boldsymbol{T}_{i,j}"] \\
\raisebox{-.5\height}{\includegraphics[width=.1\textwidth]{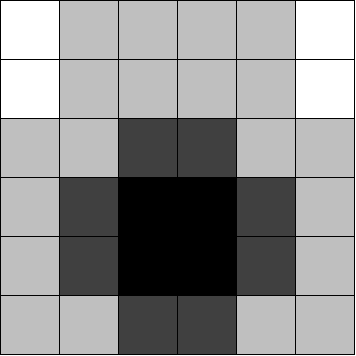}}
 \arrow[r, "\embed"] &
 {\raisebox{-.5\height}{\includegraphics[width=.1\textwidth]{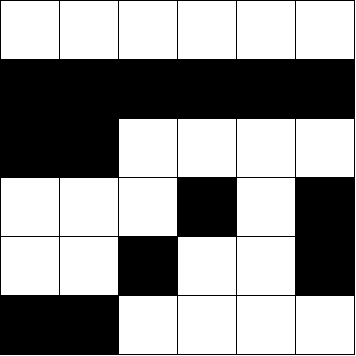}}}
\end{tikzcd}
\caption{Commutative diagram for translation-equivariance of binding operation  \eqref{2dencoding:scheme}.}
\label{fig-equi}
\end{figure}

The equivariant binding operation for the case $N = L^2$ can be extended to $N > L^2$ while maintaining exact translation-equivariance by using a block-based construction when $N = M L^2$ by taking $M$ copies of the $L \times L$ construction. More precisely, we define the 2D-Block family of permutations matrices $\{ \boldsymbol{T}_{i,j}^\text{2D-Block} : i,j \in \{1,\ldots,L\}\}$ which act on $\boldsymbol{v} \in \{-1,+1\}^{L \times L \times M}$ by
\begin{equation} \label{2dblockeq}
\left( \boldsymbol{T}_{i,j}^\text{2D-Block} \boldsymbol{v} \right)(i',j',k) = \boldsymbol v (i+ i', j+j',k),
\quad \text{for} \quad (i',j',k) \in \{1,\ldots,L\}^2 \times \{1,\ldots,M\},
\end{equation}
where the addition $i+i'$ and $j+j'$ is taken modulo $L$. Alternatively, one can consider $N = M^2 > L^2$ and define $\{ \boldsymbol{T}_{i,j}^\text{2D-Cyclic} : i,j \in \{1,\ldots,L\}\}$
\begin{equation} \label{2dcyclicshifteq}
\left( \boldsymbol{T}_{i, j}^\text{2D-Cyclic} \boldsymbol v \right)(i',j') = \boldsymbol v (i+ i', j+j'), \quad \text{for} \, (i',j') \in \{1,\ldots,M\}^2 ,
\end{equation}
where addition $i+i'$ and $j+j'$ is taken modulo $M$. When the 2D-Cyclic family of permutations is used in the binding operation, the resulting encoding is translation-equivariant when ignoring boundary effects. We demonstrate an application of the 2D-Cyclic family of permutations to visualize this equivariance property \S \ref{sec:experimentequi}.

We note that translation-equivariance has been considered in the context of Hyperdimensional computing by \cite{Rachkovskij2022,Rachkovskij2024}. In these works, a theory of an equivariant encoding scheme is developed, but no concrete examples of such a family of equivariance-permitting matrices are produced. Our work extends these results by constructing a concrete example of such a family and developing visualizations that are possible due to the similarity structure imposed on the hypervectors used in the binding operation.

\begin{remark}[Connections of constructions to neural networks]
We note two connections of the binary HDC schemes we consider to neural networks. First, the Laplace kernel, which naturally arose in our binary HDC construction, also has connections to Neural Tangent Kernels (NTKs) \cite{geifman2020similarity,chen2020deep}. Second, we note that neural networks with $+1$ and $-1$ weights and activations have been considered by several authors; see, for example, Courbariaux \emph{et al. }\cite{courbariaux2016binarized} and Kim \cite{kim2016bitwise}.
\end{remark}

\section{Experiments}

We present several numerical experiments on the binary HDC schemes described in this paper and their limitations and extensions. Code for all presented methods is publically available at: 
$$
\text{\url{https://github.com/HDStat/Laplace-HDC}}
$$
This section is organized as follows. 
In \S \ref{classmethods}, we describe the linear classifiers that we use on the binary encodings. In \S \ref{sec:vanilla}, we present results for the vanilla version of Laplace-HDC. In \S \ref{svdfeatures}, we present results for Laplace-HDC using singular value decomposition (SVD) features. In \S \ref{sec:convfeatures}, we present results for Laplace-HDC with Haar convolutional features. In \S \ref{sec:comp}, we provide a comparison of Laplace-HDC to other methods. Beyond the model accuracy, we study robustness in \S \ref{robustness}, and conduct experiments concerning the translation-equivariance property of the proposed encoding scheme in \S \ref{sec:experimentequi}.

\subsection{Classification Methods}\label{classmethods}
Once the data $\boldsymbol x\in \mathcal X$ has been embedded $\boldsymbol{x} \mapsto \boldsymbol{\psi}_{\boldsymbol{x}}$, a classifier may be trained.
Each classifier we consider uses an inner product to determine the final class.
Suppose the classes $\{1,2,\ldots, c\}$ are represented within $\mathcal X$, and suppose $\embed_i$ denotes the class representative for class $i$, for $i = 1,2,\ldots, c$ .
We determine the class  of $\boldsymbol y$ by a simple linear classifier
\begin{equation} \label{argmaxclassvec}
\mathrm{class}(\boldsymbol y) = \argmax_{i=1,...,c}\left(\embed_{\boldsymbol y}^\top\embed_{i}\right).
\end{equation}
Below, we detail four methods we use for determining the class representatives $\{\embed_{i}\}$.

\subsubsection{Float and Binary Majority Vote}\label{class:majvote}
The majority vote classifiers operate on a simple principle for classification. First, we describe the binary flavor of this method. Consider the collection
$$
C_i = \{\boldsymbol x\in \mathcal X \,\colon \text{class}(\boldsymbol x)=i\},
$$
and let $\#(C_i)$ denote the number of elements in the collection.
The representative for class $i$, denoted $\embed_{i}$, 
is determined by a majority vote of $\embed_{\boldsymbol{x}}$ for 
$\boldsymbol x\in C_i$. More precisely, we define the Binary Majority Vote representative $\embed_{i}$  by
$$
\embed_{i}(k) = \begin{cases}+1 & \sum_{\boldsymbol x \in C_i} \boldsymbol{\psi}_{\boldsymbol{x}}(k) > 0 \\ -1 &\text{otherwise,}\end{cases}
$$
The Float Majority Vote is largely the same, but we relax the condition that $\embed_i$ must have entries in ${-1, +1}$. Instead, the entries of $\embed_i$ may be floating-point numbers. In this case, the representative $\embed_i$ is determined entrywise by class averages
$$
\embed_i(k) = \frac {1}{\#(C_i)}\sum_{\boldsymbol x \in C_i} \boldsymbol{\psi}_{\boldsymbol x}(k).
$$

\subsubsection{Float and Binary SGD}\label{class:sgd}
We define two classifiers, which we call Float SGD and Binary SGD. The Float SGD classifier determines the class representatives $\embed_i$ by optimizing a cross-entropy loss function using stochastic gradient descent; more precisely, we use Adam \cite{kingma2017adam} with a learning rate parameter $\alpha=0.01$, where the model takes $\boldsymbol{\psi}_{\boldsymbol{x}}$ and outputs one of the $c$ classes. 
We perform $3$ epochs training passes over the data in $\mathcal X$ in all experiments. 
The Binary HDC classifier operates in the same manner, except during each training epoch, the weights of the model parameter are clamped to be in the range $[-1, 1]$. Once all training epochs have been completed; the $\sign$ function is applied to the model weights, making all negative weights $-1$ and all positive weights $+1$.

\subsection{Laplace-HDC in its Basic Form} \label{sec:vanilla}
In this section, we define
a Binary HDC scheme motivated by Theorem \ref{thm1}, Corollary \ref{cor1}, and Theorem \ref{admissiblefamily}, which we call Laplace-HDC. Given a data set $\mathcal{X} = \{\boldsymbol{x}_1,\ldots,\boldsymbol{x}_n\}$ where $\boldsymbol{x}_i \in \{1,\ldots,m\}^d$, Laplace-HDC consists of the following steps:

\begin{enumerate}[label=\arabic*.]
\item Set the bandwidth parameter $\lambda > 0$ using
Remark \ref{rem:settingbandwidth}
with $\alpha =1$.
\item Define the kernel $\boldsymbol{K} \in \mathbb{R}^{m \times m}$ using 
\eqref{admissiblekerneleq}.
\item Construct hypervectors $\boldsymbol{v}_1,\ldots,\boldsymbol{v}_m \in \{-1,+1\}^N$ using the method detailed in \S \ref{definehypervec}.
\item Choose a family of permutation matrices $\mathcal{P} = \{\boldsymbol{\Pi}_1,\ldots,\boldsymbol{\Pi}_d\}$ that are Trace-Orthogonal in the sense of Assumption \ref{assumenooverlap}.
\item Encode each $\boldsymbol{x}_j \mapsto \embed_{\boldsymbol{x}_j}$ using the binding operation \eqref{generalencoding:scheme2}.
\item Return $\{\boldsymbol{\psi}_{\boldsymbol{x}_1},
\ldots,\boldsymbol{\psi}_{\boldsymbol{x}_n}\}$ 
where $\boldsymbol{\psi}_{\boldsymbol{x}_i} \in \{-1,+1\}^N$.
\end{enumerate}
By construction, it follows from Theorem \ref{admissiblefamily} that
$$
\mathbb{E} \frac{\boldsymbol{\psi}_{\boldsymbol{x}_i}^\top \boldsymbol{\psi}_{\boldsymbol{x}_j}}{N} 
\approx \exp( -\lambda \|\boldsymbol{x}_i - \boldsymbol{x}_j\|_1 ).
$$
In the following, we demonstrate the utility of Laplace-HDC in an application to image classification. We consider four choices of trace-orthogonal families of permutation matrices: 
1D-Cyclic shift, 1D-Block Cyclic, 2D-Block Cyclic, and 2D-Cyclic Shift, which are defined in 
\eqref{cyclicshiftseq}, \eqref{1dblockeq}, \eqref{2dblockeq}, and \eqref{2dcyclicshifteq}, respectively. Classification is performed by determining which class representative $\boldsymbol{\psi}_i$ gives the largest inner product \eqref{argmaxclassvec}. The class representatives $\boldsymbol{\psi}_i$ are determined by two different methods: Float SGD and Binary SGD, which are defined in \S \ref{class:sgd}.

In each case, we use the largest possible hyperdimensionality $N \le 10^4$. For example, the 1D-Block family of permutations requires that $N = d M$ for some positive integer $M$. The image data we consider has dimension $d = 28^2$, so we choose $M = \lfloor 10^4/28^2 \rfloor =12$, which results in $N = 9408$. When setting the bandwidth parameter $\lambda > 0$, we use \eqref{bandwidthset} with $c=1$ except for binary SGD where  $c=4$ provides better performance, and we estimate the median of the $\ell_1$-distances of the data using $1000$ samples selected uniformly at random from $\mathcal X$.  The accuracy for each is presented as the mean $\pm$ one standard deviation in all cases; see Table \ref{tab:vanilla}.

\begin{table}[ht!]
\caption{Basic Laplace-HDC Performance}

\begin{center}
\begin{NiceTabular}[corners=NW,hvlines]{ccccc}
\CodeBefore
  \rowcolor{gray!20}{1}
  \columncolor{gray!20}{1}
  \rowcolors{2}{gray!0}{gray!0}[cols={2,3,4,5,6}]
\Body
               & 1D-Cyclic & 2D-Cyclic &1D-Block  & 2D-Block \\
\!\!\!Float SGD - Fashion MNIST\!\!\! &   $86.06\pm 0.74$ & $86.03\pm 0.73$ & $87.86\pm 0.40$  & $87.41\pm 0.15$   \\
\!\!\!Binary SGD - Fashion MNIST\!\!\! &   $83.60\pm 1.35$& $83.87\pm 0.63$ &$84.72\pm 0.23$ & $83.51\pm 0.57$     \\
\!\!\!Float SGD - MNIST\!\!\! & $95.46\pm0.48$ & $95.46\pm 0.51$ & $96.13 \pm 0.25$ & $96.14 \pm 0.28$\\
\!\!\!Binary SGD - MNIST\!\!\! & $93.25\pm 1.09$ & $93.37\pm 1.00$ & $94.43 \pm 0.82$ & $94.59 \pm 0.44$\\
\end{NiceTabular}
\end{center}\label{tab:vanilla}
\end{table}

\subsection{Laplace-HDC with SVD Features} \label{svdfeatures}
Singular value decomposition is a popular pre-processing tool in predictive tasks involving numerical features. Well-known techniques such as principal component regression use this decomposition to rotate the coordinate system so that the features are uncorrelated in the rotated system. The process normally involves populating the numerical features into a matrix $\boldsymbol{X}$ (where the number of columns corresponds to the number of features and the number of rows corresponds to the number of samples) and then performing a compact SVD of the form  $\boldsymbol{X} = \boldsymbol{U}\boldsymbol{\Sigma}\boldsymbol{V}^\top$ to acquire the transformation matrix. The new feature matrix takes the form $\hat{\boldsymbol{X}} = \boldsymbol{X}\boldsymbol{V}$, which not only enjoys uncorrelated features, also may have fewer columns than those in $\boldsymbol{X}$ thanks to the compact nature of the SVD operation (instead of a full SVD). Truncating small singular values is a manual alternative to reduce the number of features, which is also capable of denoising the feature matrix. This technique is incorporated into the HDC pipeline by first mapping the data matrix $\boldsymbol{X}$ to  $\hat{\boldsymbol{X}}$ and then performing the hyperdimensional encoding on $\hat{\boldsymbol{X}}$.

To evaluate the effects of SVD transformation, we applied it to the FashionMNIST data and used a total of 8 variations of permutation schemes and classifier pairs.  The mean test accuracy (in percentage) of different classifiers, plus and minus one standard deviation, is computed for 50 independent experiments after SVD preprocessing and different permutation schemes. The results are available in Table \ref{tab:svd}. The hyperparameters $\lambda$ and $N$ are set in the same way described in \S \ref{sec:vanilla}.

\begin{table}[ht!]
\caption{Laplace-HDC with SVD features}

\begin{center}
\begin{NiceTabular}[corners=NW,hvlines]{ccc}
\CodeBefore
  \rowcolor{gray!20}{1}
  \columncolor{gray!20}{1}
  \rowcolors{2}{gray!0}{gray!0}[cols={2,3,4,5,6}]
\Body
               &1D-Cyclic  & 1D-Block \\
Float SGD - Fashion MNIST &    $87.26\pm 0.16$  & $86.67\pm 0.44$   \\
Binary SGD - Fashion MNIST &    $84.40\pm 0.42$ & $83.57\pm 1.13$     \\
Float SGD - MNIST & $95.63 \pm 0.30$ & $96.15 \pm 0.30$\\
Binary SGD - MNIST & $94.03 \pm 0.35$ & $94.66 \pm 0.39$\\

\end{NiceTabular}
\end{center}\label{tab:svd}
\end{table}

\subsection{Laplace-HDC with Haar Convolutional Features} \label{sec:convfeatures}

The result of 
Theorem \ref{thm1}, Corollary \ref{cor1}, and Theorem \ref{admissiblefamily} suggest a limitation of using the inner product similarity structure of HDC encodings when applied to images: the spatial relationship between pixels is lost. In images, the spatial relation between pixels contains meaningful information. 
For example, a group of black pixels may indicate structure, such as a digit, while scattered black pixels may be noise. Note that the $\ell_1$-norm is invariant to permutations
$$
\|\boldsymbol{x}\|_1 = \sum_{i=1}^d |\boldsymbol{x}(i)| = \sum_{i=1}^d |\boldsymbol{x}(\sigma(i))| = \|\boldsymbol{x} \circ \sigma\|_1,
$$
for any fixed permutation $\sigma$ of $\{1,\ldots,d\}$. Thus, by Theorem \ref{thm1}, the expected similarity structure of the HDC embedding is invariant to a global permutation of the pixels of all the images. One basic way to encode spatial information is to use convolutional features.  As a basic demonstration, we consider the $9$ Haar wavelet matrices of dimension $4 \times 4$; see Figure \ref{figconv}.

\begin{figure}[ht!]
\centering
\includegraphics[width=.45\textwidth,trim={3cm 1.5cm 3cm 0.5cm}]{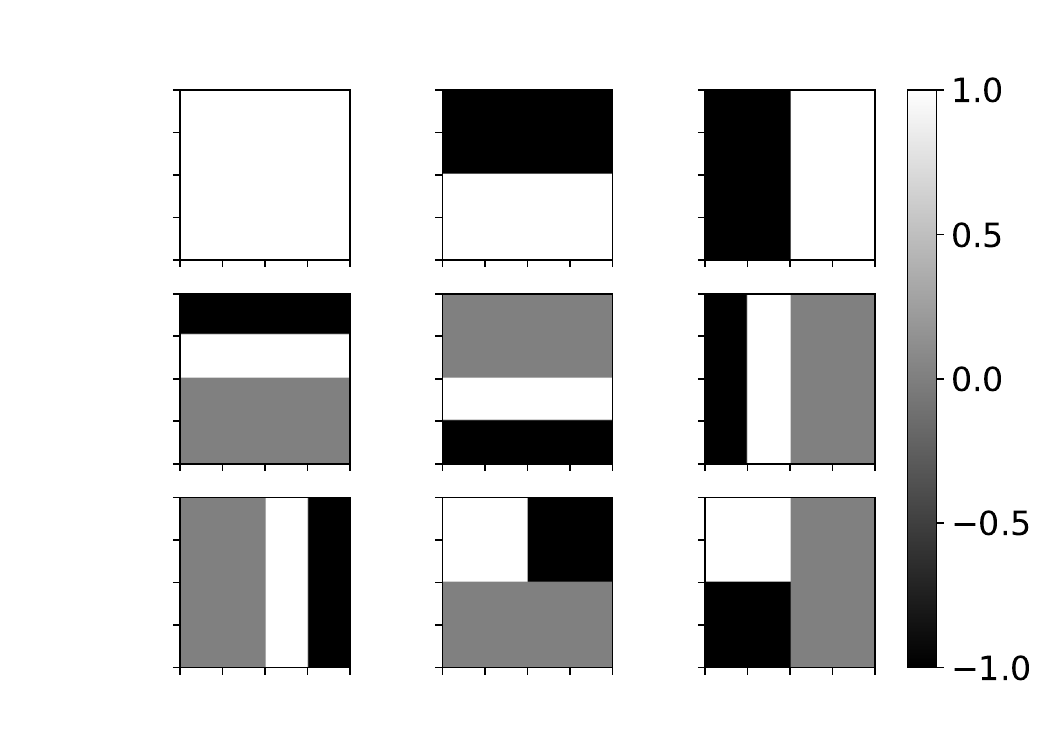}
\caption{Collection of $9$ Haar convolution matrices
 of dimension $4 \times 4$.} \label{figconv}
\end{figure}

Convolving these $9$ filters of dimension $4 \times 4$  with an $L \times L$ image with stride $s$ creates 
$$
n = 9 ((L - 4)/s + 1)^2
$$
convolutional Haar features, where the stride is the amount of each filter is shifted in each direction when convolving over the data. Each feature coordinate is mapped to the interval $[0,255]$ by an affine transformation (determined from the training data), which is rounded to an integer in $\{0,1,\ldots,255\}$. 
These integer features are used in the same Laplce-HDC methodology described in \S \ref{sec:vanilla}; the results are reported in  Table \ref{tab:conv}.

\begin{table}[ht!]
\caption{Laplace-HDC with Haar convolutional features}

\begin{center}
\begin{NiceTabular}[corners=NW,hvlines]{ccc}
\CodeBefore
  \rowcolor{gray!20}{1}
  \columncolor{gray!20}{1}
  \rowcolors{2}{gray!0}{gray!0}[cols={2,3,4,5,6}]
\Body
               &1D-Cyclic  & 1D-Block \\
Float SGD - Fashion MNIST &    $88.67\pm 0.30$  & $87.86\pm 0.40$   \\
Binary SGD - Fashion MNIST &    $86.65\pm 0.36$ & $85.63\pm 0.56$     \\
Float SGD - MNIST & $96.40 \pm 0.28$ & $96.22 \pm 0.29$\\
Binary SGD - MNIST & $95.17 \pm 0.44$ & $94.85 \pm 0.50$\\

\end{NiceTabular}
\end{center}\label{tab:conv}
\end{table}

More generally, using features from a trained convolutional neural network is possible and would improve the accuracy even further; see \S \ref{discussion} for further discussion.

\subsection{Comparison to Other Methods} \label{sec:comp}

In this section, we implement our proposed framework and compare it to some of the relevant works in the literature, such as  RFF-HDC, OnlineHD, and (Extended) HoloGN. In the sequel, we first briefly overview each of these methods and then report their performance on standard datasets such as MNIST and FashionMNIST. To present the results in a reliable format, the experiments are performed 50 times, and mean accuracies along with standard deviation of the accuracies and the histograms are reported in Figure \ref{figRealData_all} and Table \ref{tab:comp}. 

\paragraph{\textbf{RFF-HDC}}
This method, presented by Yu \emph{et al.} \cite{yu2022understanding}
 is the most relevant baseline, as it is the basis for our work. While traditional hyperdimensional computing methods (here referred to as Vanilla HDC) are fast, they suffer from low prediction accuracy. RFF-HDC utilizes similarity matrices in a similar way to Random Fourier Features (RFF) to construct the hypervectors, which helps outperform the state-of-the-art HDC schemes. A pseudo-code that constructs the hypervectors for RFF-HDC 
was given in \S \ref{sec:motivation}. Note that all the floating point operations are performed during the construction of the hypervectors or learning the models, and ultimately, the models are fully represented and operated in binary mode for inference.

\begin{figure*}[!ht]\centering 
\begin{overpic}[trim={025 -.25cm  0 0},clip,height=1.4in]{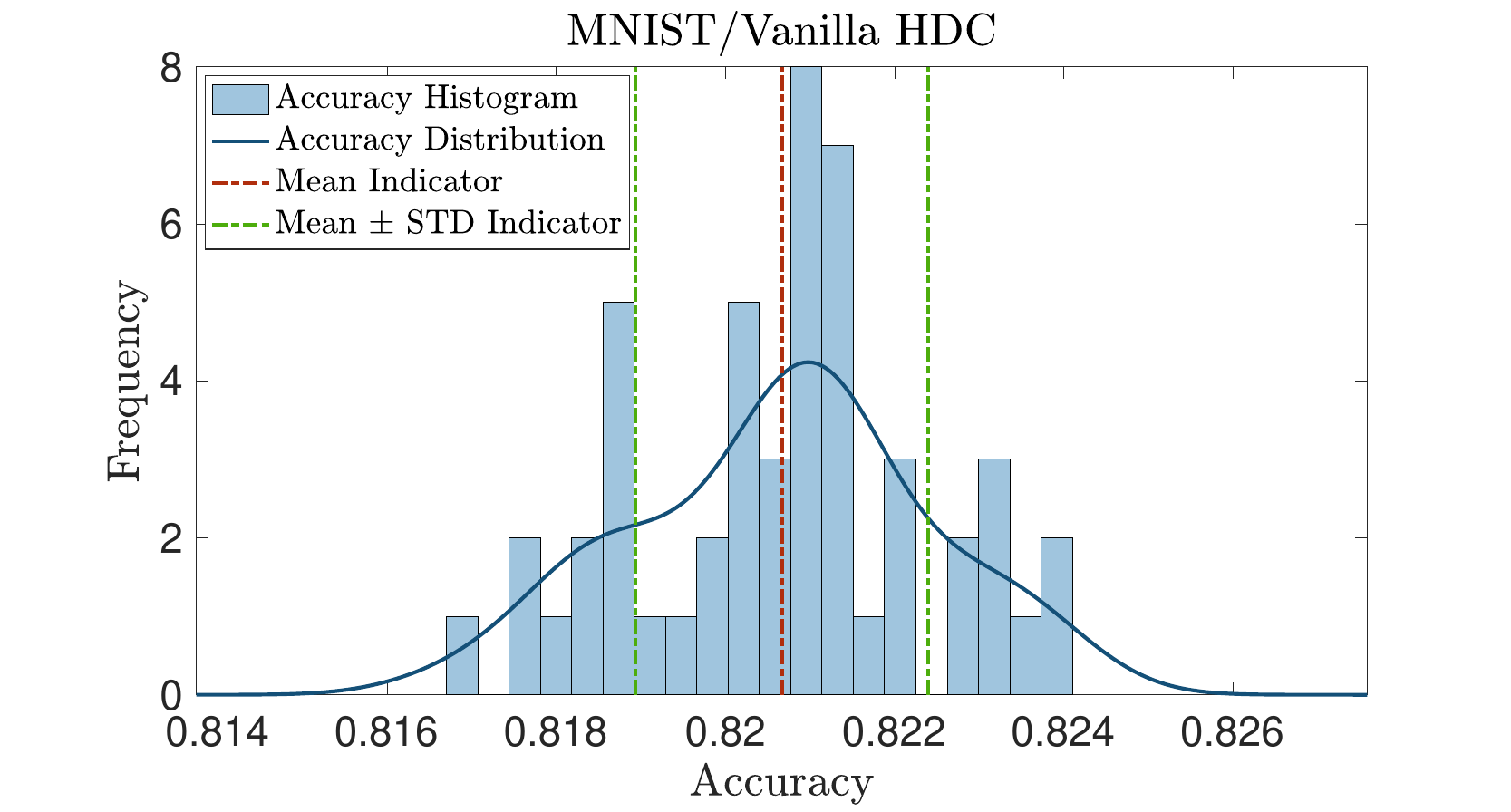}
\end{overpic}\hspace{-.6cm}
\begin{overpic}[trim={025 -.25cm  2.5cm 0},clip,height=1.4in]{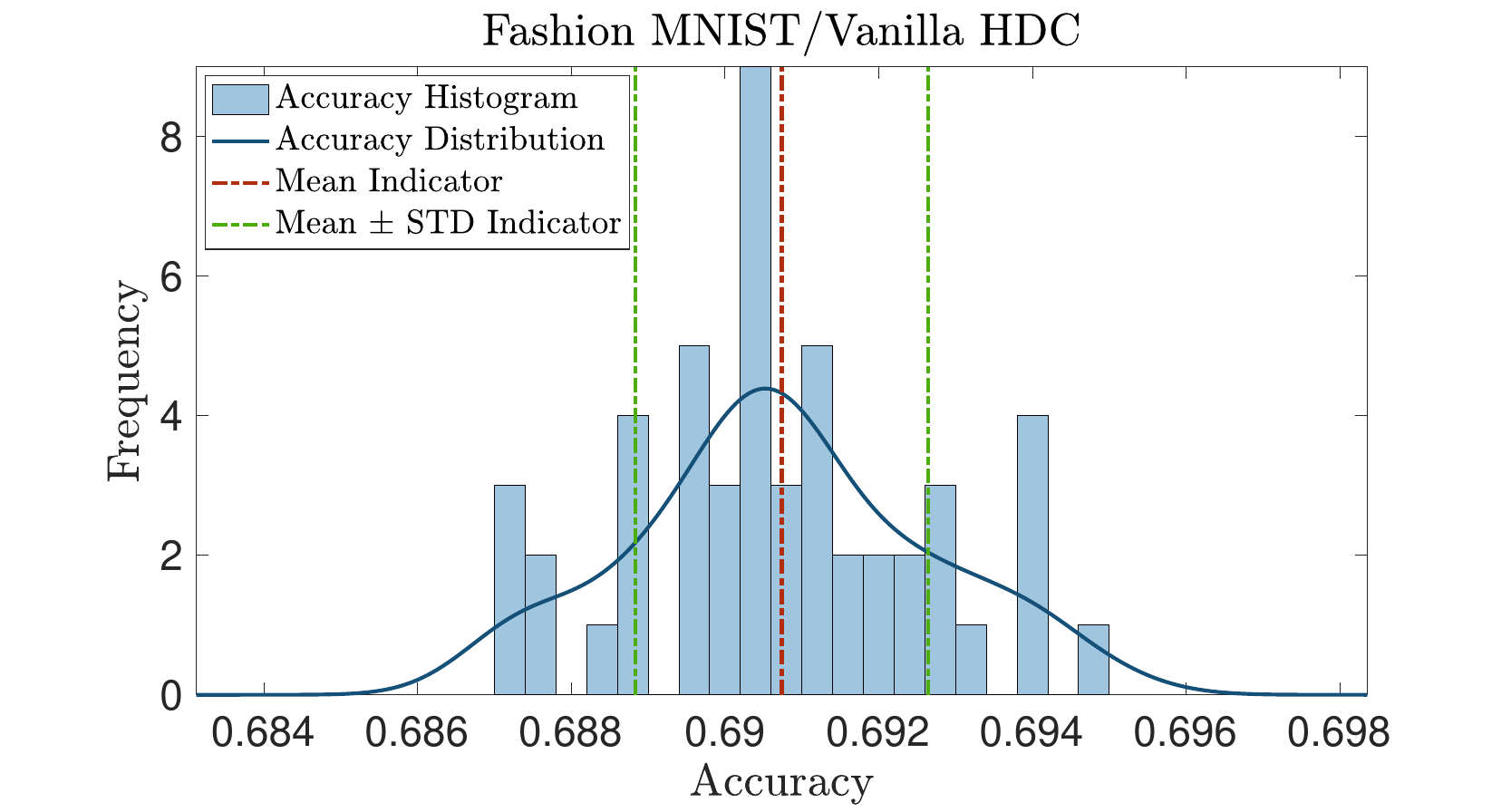}
\put (-1,1) {\scalebox{.75}{\rotatebox{0}{(a)}}} 
\end{overpic}\\[.01cm]
\begin{overpic}[trim={025 -.25cm  0 0},clip,height=1.4in]{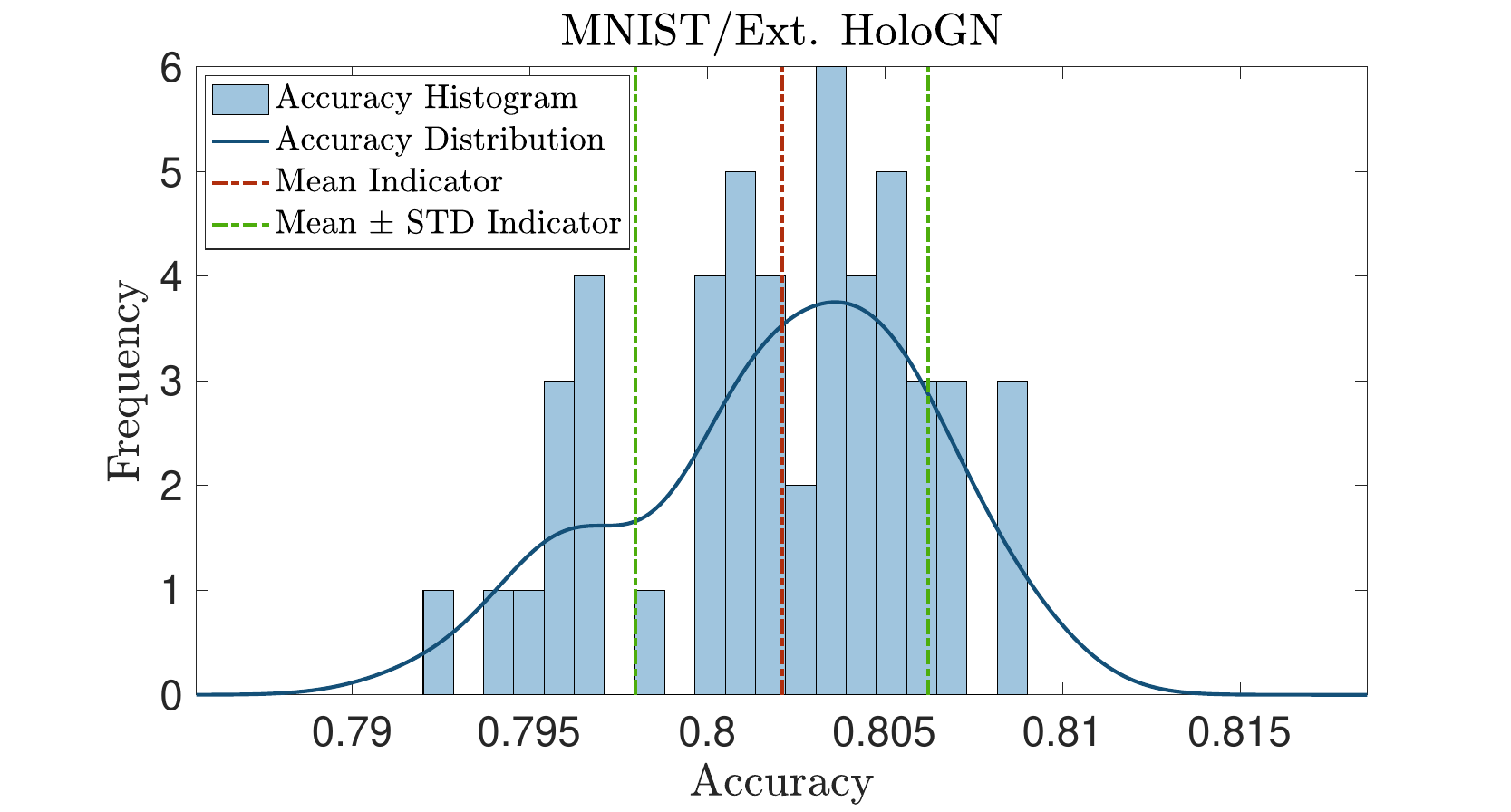}
\end{overpic}\hspace{-.6cm}
\begin{overpic}[trim={025 -.25cm  2.5cm 0},clip,height=1.4in]{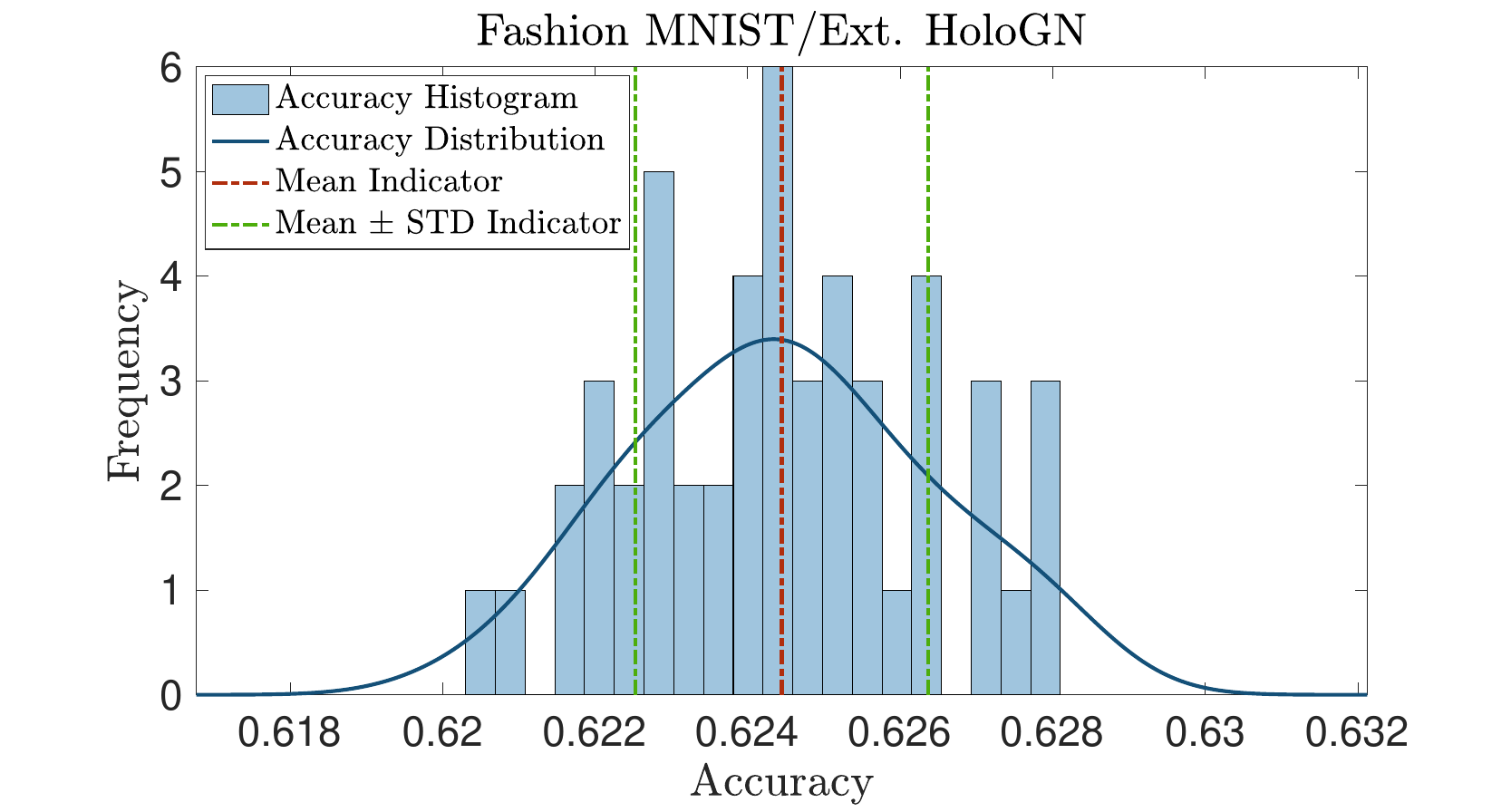}
\put (-1,1) {\scalebox{.75}{\rotatebox{0}{(b)}}} 
\end{overpic}\\[.01cm]
\begin{overpic}[trim={025 -.25cm  0 0},clip,height=1.4in]{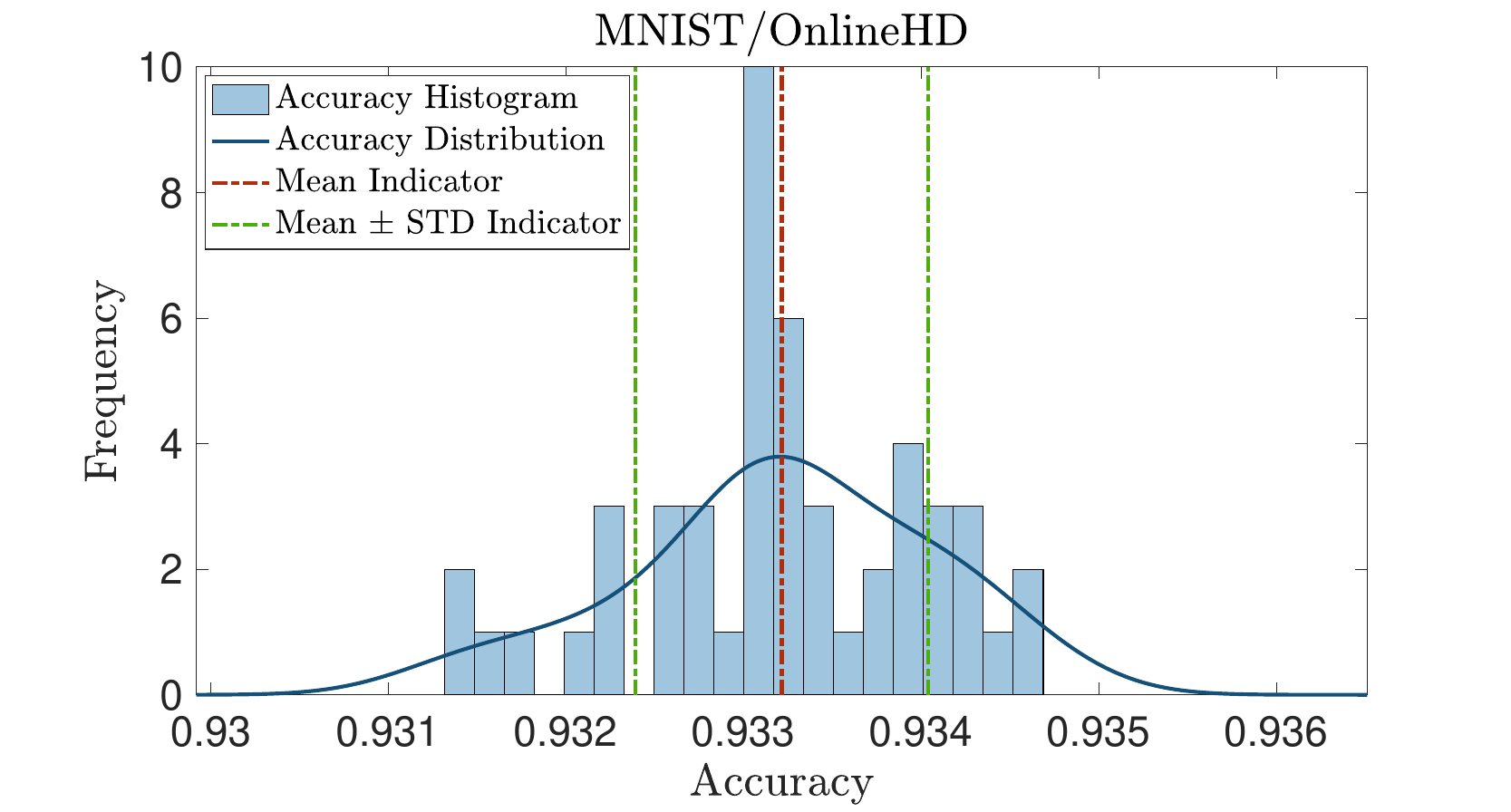}
\end{overpic}\hspace{-.6cm}
\begin{overpic}[trim={025 -.25cm  2.5cm 0},clip,height=1.4in]{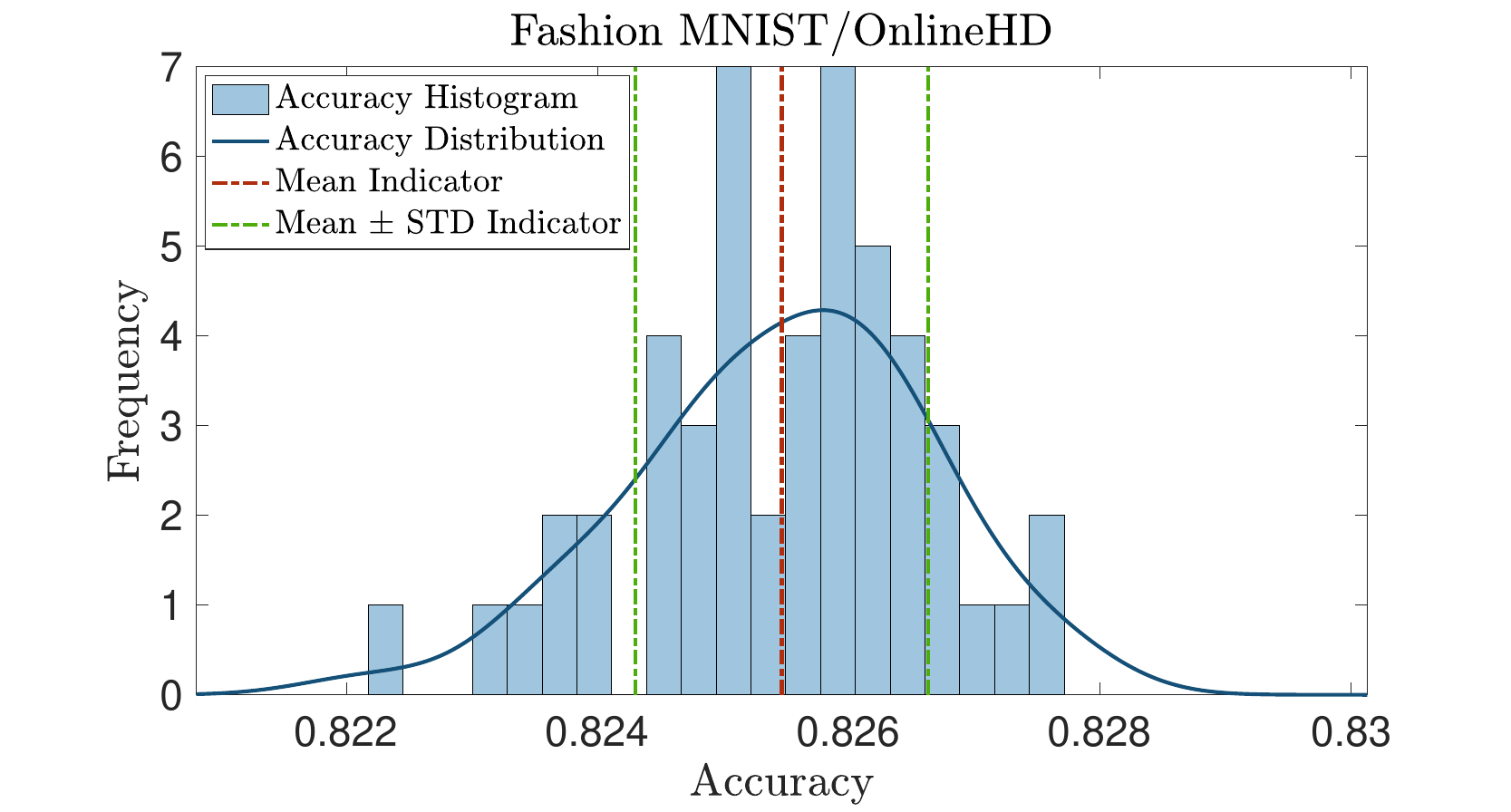}
\put (-1,1) {\scalebox{.75}{\rotatebox{0}{(c)}}} 
\end{overpic}\\[.01cm]
\begin{overpic}[trim={025 -.25cm  0 0},clip,height=1.4in]{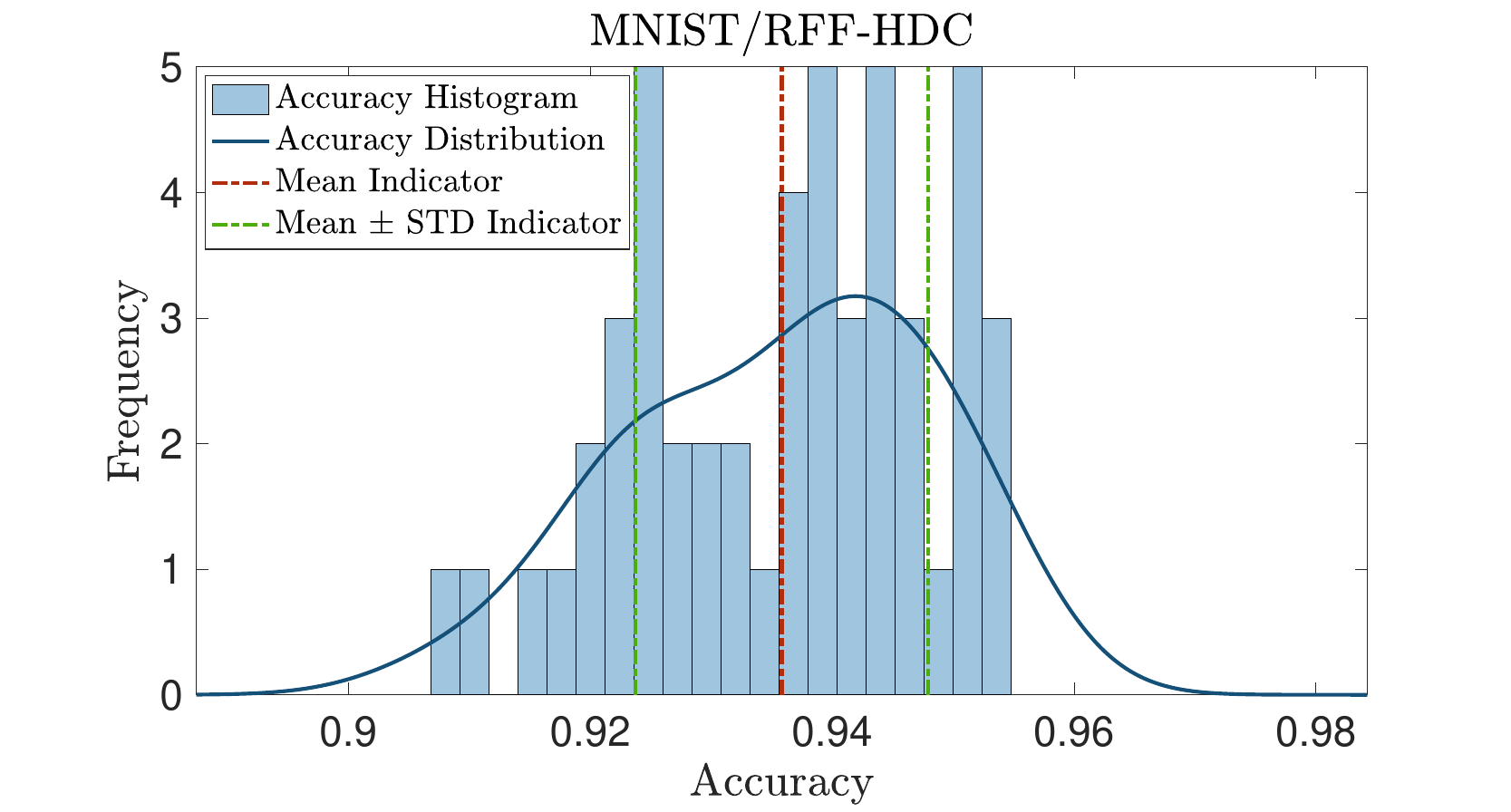}
\end{overpic}\hspace{-.6cm}
\begin{overpic}[trim={025 -.25cm  2.5cm 0},clip,height=1.4in]{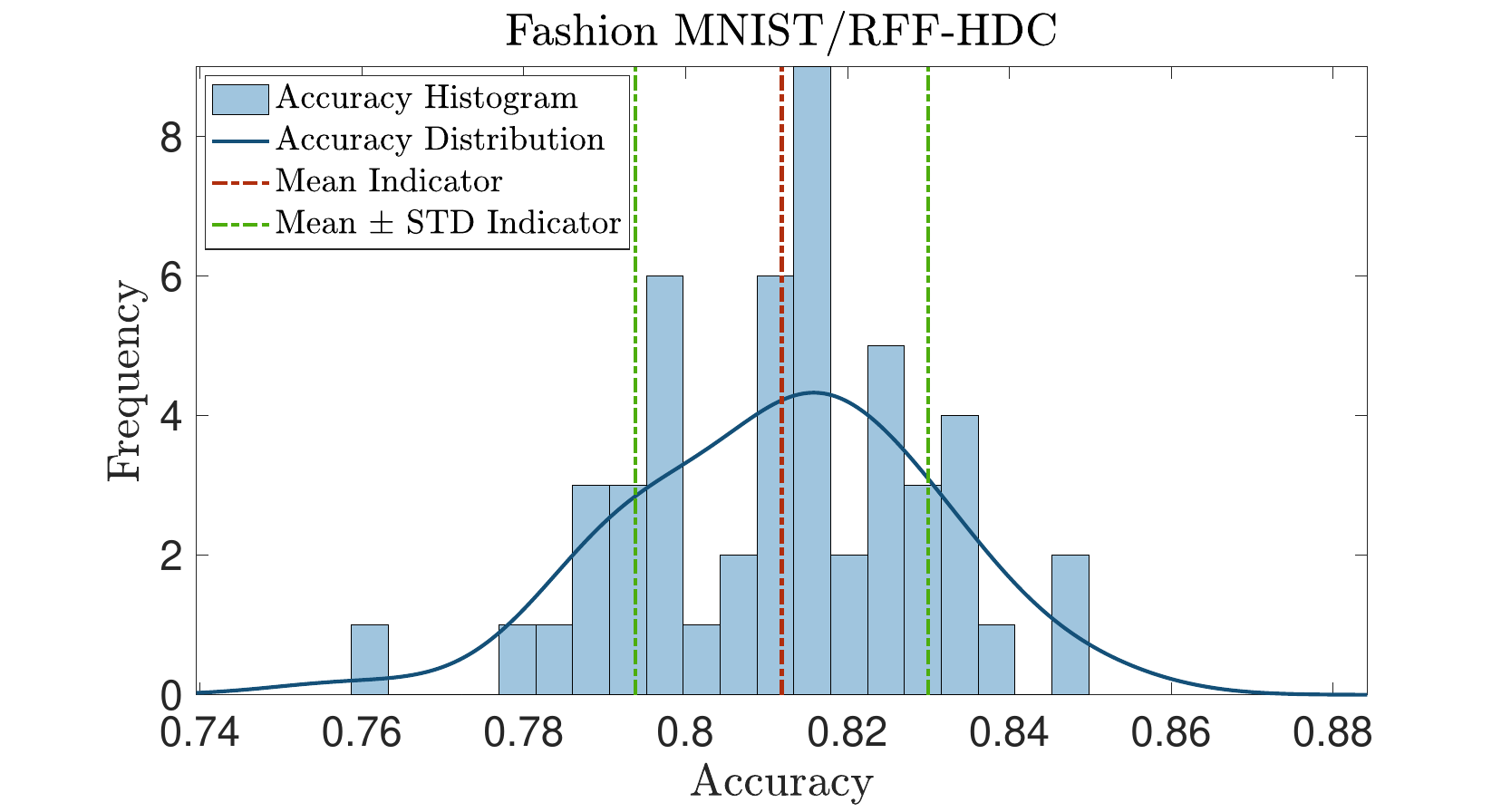}
\put (-1,1) {\scalebox{.75}{\rotatebox{0}{(d)}}} 
\end{overpic}\\[.01cm]
\begin{overpic}[trim={025 -.25cm  0 0},clip,height=1.4in]{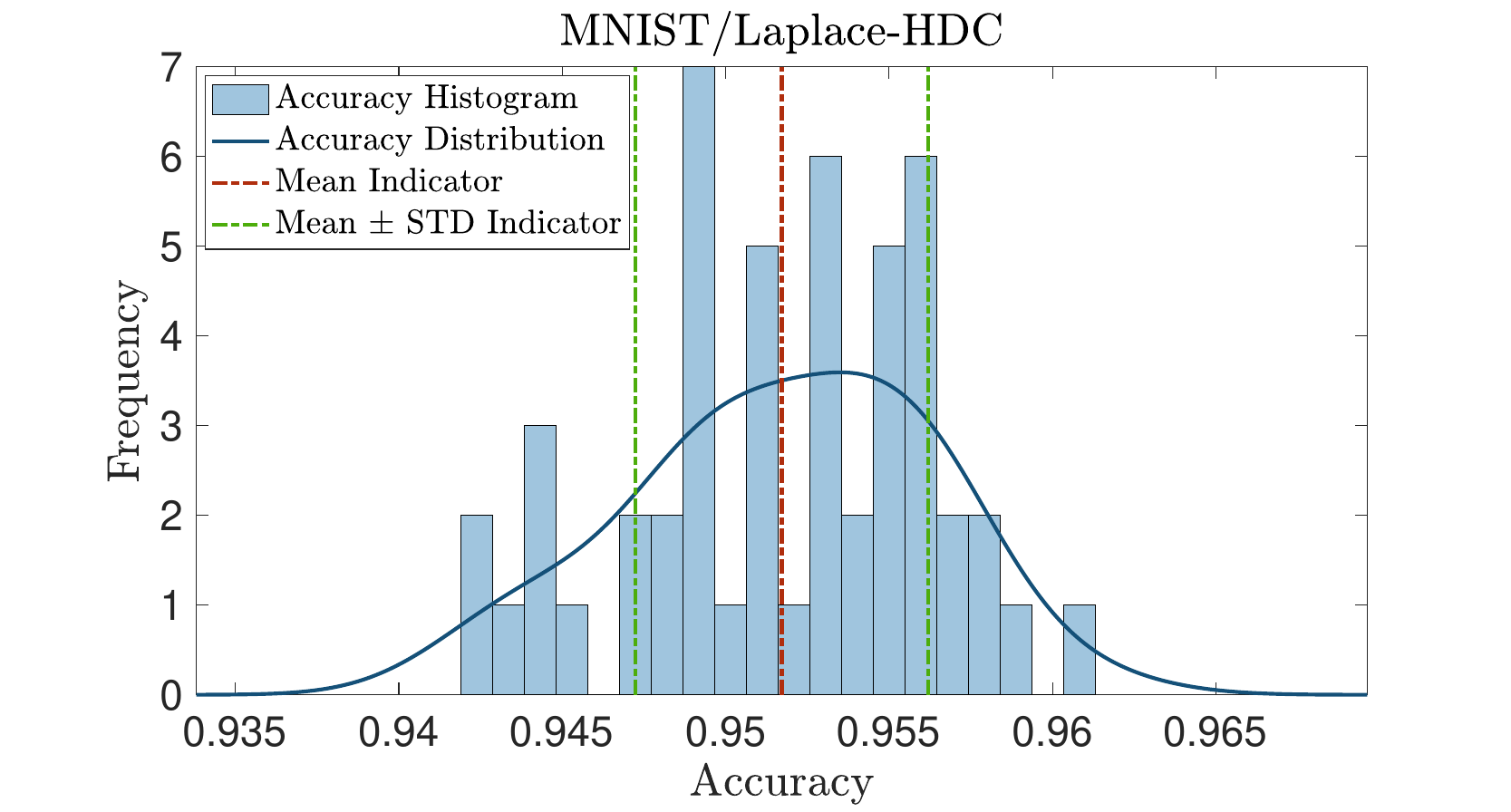}
\end{overpic}\hspace{-.6cm}
\begin{overpic}[trim={025 -.25cm  2.5cm 0},clip,height=1.4in]{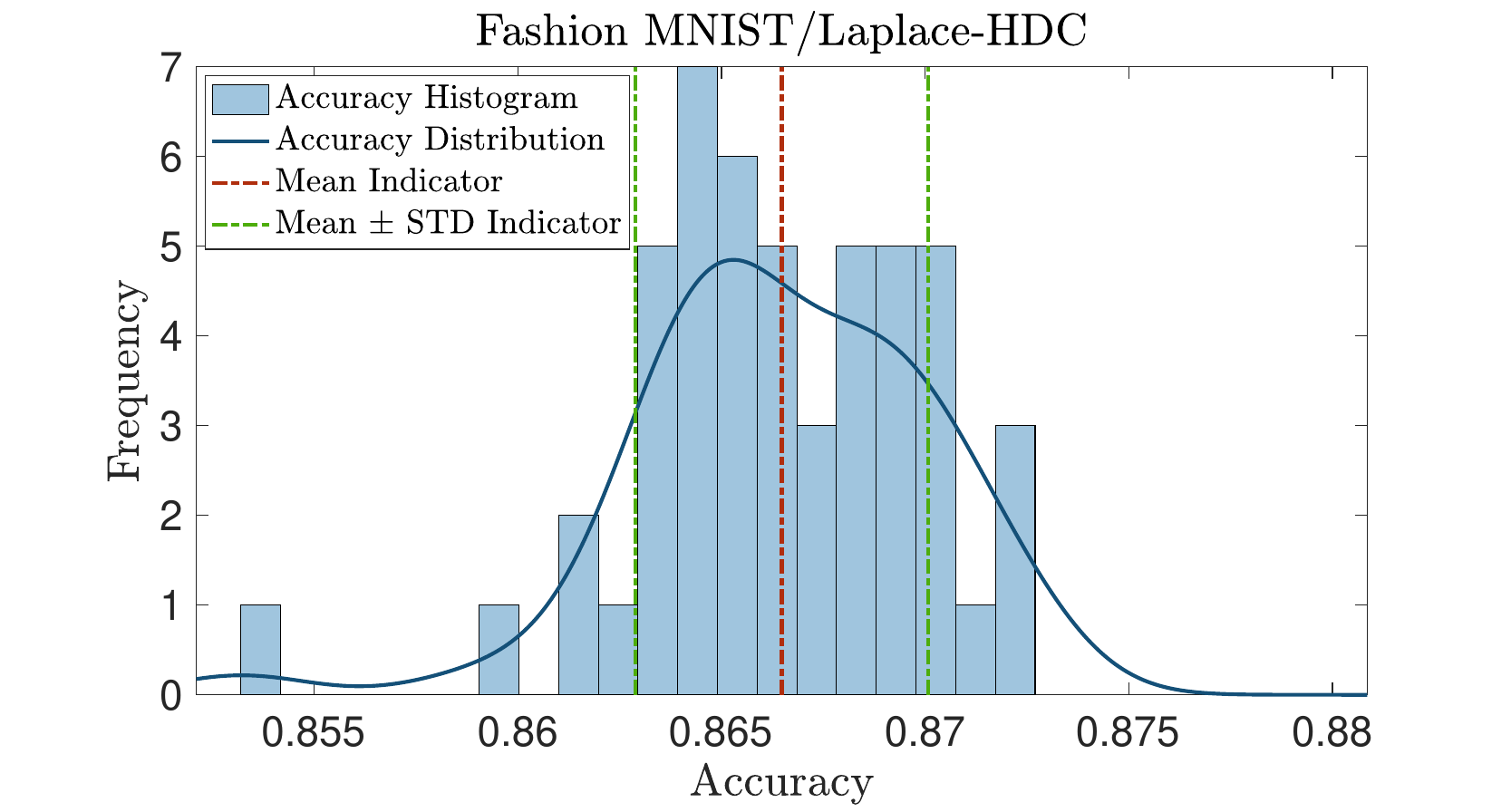}
\put (-1,1) {\scalebox{.75}{\rotatebox{0}{(e)}}} 
\end{overpic}\\[.01cm]
\caption{Accuracy histograms of the methods in Table \ref{tab:comp} for 50 trials. The mean accuracy and one-standard deviation interval are shown with dashed lines: (a) Vanilla HDC, (b) Extended HoloGN, (c) OnlineHD, (d) RFF-HDC, (e) Laplace-HDC }\label{figRealData_all}\vspace{-.5cm}
\end{figure*}

\paragraph{\textbf{OnlineHD}}
Performing iterative training rather than single-pass training is one approach to boost the accuracy of HDC models, although it increases the time complexity and memory usage, which is costly. 
Methods such as OnlineHD \cite{onlineHD2021} relate the low accuracy of single-pass models to the naive gathering of information from all hypervectors that belong to the same class. This leads to the dominance of the common pattern while downplaying the more uncommon patterns in the data. OnlineHD \cite{onlineHD2021} is presented as a single-pass remedy to this problem. Basically, if a hypervector is closely similar to the current state of the class hypervector, then OnlineHD assigns a small weight to it while updating the model in order to decrease its effect, and if the hypervector is distant, the weight increases. 

\paragraph{\textbf{Extended HoloGN}}
Holographic graph neuron (HoloGN) \cite{KleykoDenis2017HGNA} is an approach designed for character recognition over a dataset of small binary images. HoloGN assigns a randomly generated hypervector to each pixel. Then, a circular shift occurs if the pixel color is white. After this stage, the remaining procedure is similar to vanilla HDC bundling (Binary Majority vote, see \S \ref{class:majvote}). An extension of this approach to operate with non-binary datasets such as MNIST was presented in Manabat \emph{et al.} \cite{ExHoloGN}, which is used as a comparison baseline in our experiments. 

The accuracies reported in Table \ref{tab:comp} and the histograms depicted in Figure \ref{figRealData_all} show that Laplace-HDC with convolutional features can outperform the state-of-the-art techniques in terms of the mean accuracy. In terms of the accuracy standard deviation, Laplace-HDC offers a significantly lower deviation compared to the RFF-HDC, with which it shares some foundations. 

\begin{table}[ht!]
\caption{The mean test accuracy (in percentage) of different methods discussed in \S \ref{sec:comp} plus and minus one standard deviation, computed for 50 independent experiments. The reported Laplace-HDC uses Haar convolutional features, 1D-Cyclic permutations, and the Binary SGD classifier.}
\centering
\begin{NiceTabular}[corners=NW,hvlines]{cccccc}
\CodeBefore
  \rowcolor{gray!20}{1}
  \columncolor{gray!20}{1}
  \rowcolors{2}{gray!0}{gray!0}[cols={2,3,4,5,6}]
\Body
               & \!\!\!Vanilla HDC\!\!\! &\!\!\! Ext. HoloGN \!\!\! &\!\!\! OnlineHD\!\!\!  & \!\!\!RFF-HDC\!\!\! &\!\!\! Laplace-HDC\!\!\! \\
\!\!MNIST\!\! &    \!\!\!$82.07\pm0.17$\!\!\! &  \!\!\!$80.21\pm 0.41$\!\!\!  &\!\!\! $93.32\pm 0.08$ \!\!\! & \!\!\!$93.58\pm 1.21$\!\!\! &\!\!\! $95.17\pm 0.44$ \!\!\!    \\
\!\!Fashion MNIST\!\! &\! \!\! $69.07\pm0.19$ \!\!\!& \!\!\!$62.44\pm 0.19$\!\!\! &\!\!\! $82.55\pm 0.12$\!\!\! &\!\!\! $81.19\pm 1.81$\!\!\!  & \!\!$86.65\pm 0.36$ \!\!\!      \\
\end{NiceTabular}
\label{tab:comp}
\end{table}

In the next two sections, we explore some other aspects of the proposed binary HDC beyond the accuracy. 

\subsection{Robustness to Corruptions}\label{robustness}
A notable characteristic of binary HDC 
encodings are their robustness to noise: these binary encodings can remain effective even in the presence of corrupted bits. To demonstrate this robustness property for Laplace-HDC, we perform an experiment where the classification task is stress-tested by randomly corrupting a proportion of the bits of the binary encoding.

In this experiment, for each $\boldsymbol{x} \in \mathcal{X}$ the corresponding encoded hypervector $\embed_{\boldsymbol{x}}\in\{-1,1\}^N$ is corrupted by flipping $k$ of the $N$ bits of the encoded vector to generate a corrupted hypervector $\widetilde{\embed}_{\boldsymbol x}$. More precisely, for each $\boldsymbol{x} \in \mathcal{X}$ we choose a set $\{i_1,\ldots,i_k\}$ from 
$\{1,\ldots,N\}$ independently and uniformly at random without replacement and set
$$
\widetilde{\boldsymbol{\psi}}_{\boldsymbol{x}}(i_j) = -\boldsymbol{\psi}_{\boldsymbol{x}}(i_j), \quad \text{for} \quad j =1,\ldots,k.
$$
After this step, the class of $\widetilde{\embed}_{\boldsymbol x}$ is inferred using a classifier model which was trained on the uncorrupted data $\embed_{\boldsymbol x}$. In this fashion, we are able to determine the degree to which corruption affects classification accuracy. In this experiment, we use the Binary SGD classifier, see \S \ref{class:sgd}. We report results based on the ratio $k/N$ of corrupted bits to the bits in the encoding.
When $k/N=0$, the embedding  $\embed_{\boldsymbol{x}}$ is not altered; the method and classification accuracy are the same as reported in \S \ref{sec:vanilla}.
When $k/N = 1/2$, half the bits are randomly corrupted, and the corrupted embedding $\widetilde{\embed}_{\boldsymbol x}$ and original embedding $\embed_{\boldsymbol x}$ become uncorrelated. 
Figure \ref{fig:robustness} shows the degradation pattern of the Laplace-HDC accuracy as a function of bit error rate. The experiment is performed for the FashionMNIST data for $k=0,2,...,5000$, and hyperparameters $\lambda$ and $N$ are set in the same way described in \S \ref{sec:vanilla}. Each experiment is performed multiple times, and in addition to the mean accuracy, an uncertainty region of radius three standard deviations is depicted around the mean accuracy plot. One can see that with almost up to a 25\% bit error rate (which corresponds to 50\% of the worst possible corruption), the classification accuracy confidently maintains a value above 80\%.

\begin{figure}[ht!]
\centering
\begin{overpic}[trim={025 -.25cm  2.5cm 0},clip,width=.8\textwidth]{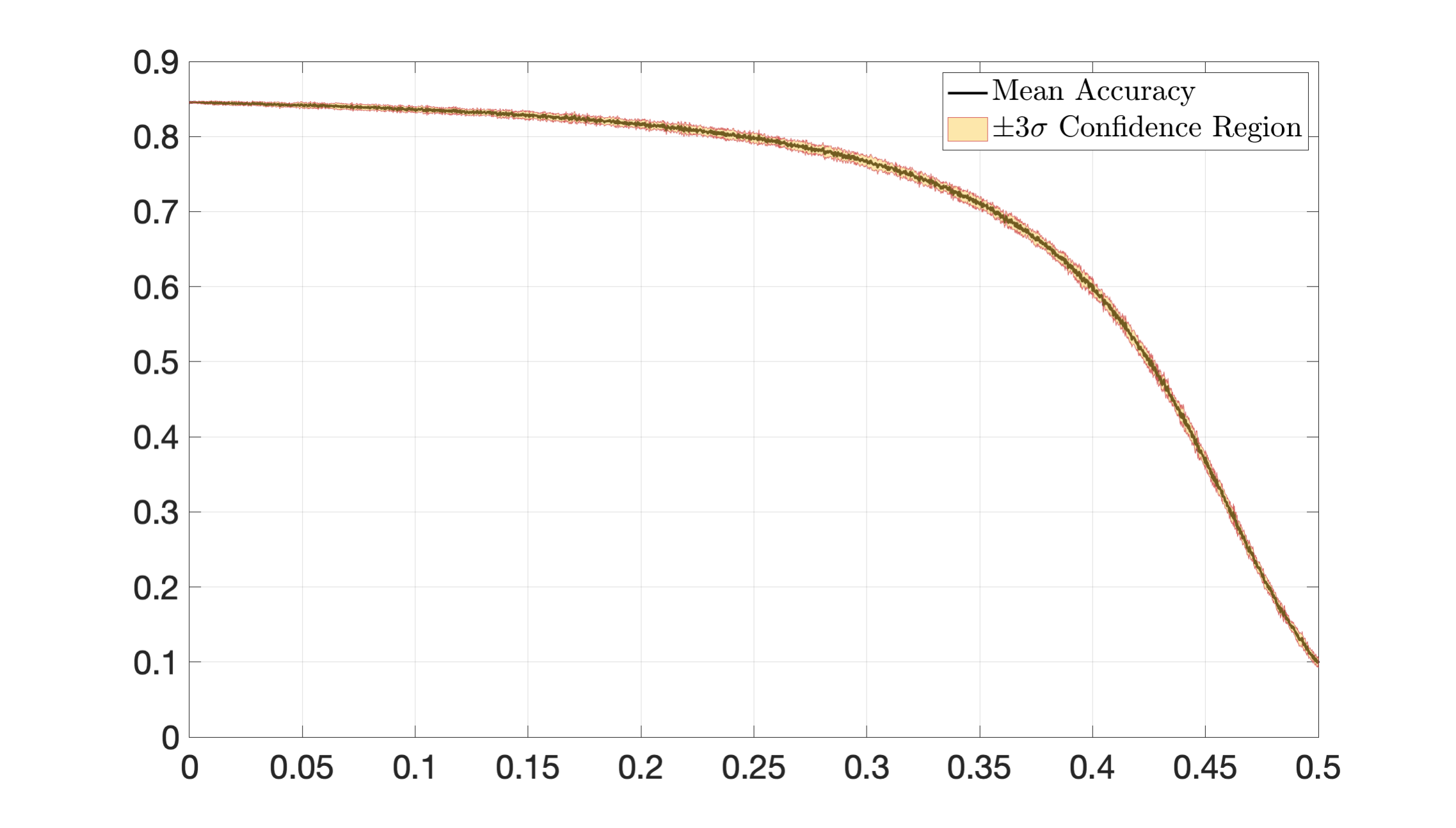}
\put (40,-1) {\scalebox{.75}{\rotatebox{0}{Ratio of Randomly Flipped Bits}}} 
\put (2,25) {\scalebox{.75}{\rotatebox{90}{Model Accuracy}}} 
\end{overpic}
\caption{The robustness of the proposed HDC formulation to noise: in each hypervector, a portion of the bits are randomly flipped, and the accuracy is evaluated for the noisy model. The narrow shaded region around the accuracy curve is $\pm$3 times the standard deviation of the accuracy. One can notice that up to almost 25\% bit error rate, the accuracy does not drop below 80\%.}\label{fig:robustness}
\end{figure}

\subsection{Translation Equivariance}
\label{sec:experimentequi}

In this section, we describe how the 2D-Cyclic family of 
permutations defined in \eqref{2dcyclicshifteq} encodes spatial information and leads to interesting visualizations. In particular, we consider the encoding $\{1,\ldots,m\}^{L \times L} \rightarrow \{-1,+1\}^{M \times M}$ by 
$$
\boldsymbol x\mapsto \embed_{\boldsymbol{x}} = \bigodot_{i,j=1}^{M} 
\boldsymbol{T}_{i, j}^\text{2D-Cyclic}
\boldsymbol v_{\boldsymbol x(i, j)},
$$
which is translation-equivariant up to boundary effects, see \S \ref{sec:translationequivar}. Consider examples from the FashionMNIST dataset; see Figure \ref{fig:fashionexamples}.

Let $\boldsymbol{z}$ denote the all zero image $\boldsymbol{z}(i,j) = 0$ for $i,j =1,\ldots,L$, and $\boldsymbol{\psi}_{\boldsymbol{z}}$ be the encoding of the all zero image. To visualize the hypervector encodings of these images, we plot 
$$
\boldsymbol{\psi}_{\boldsymbol{x}_1} \odot \boldsymbol{\psi}_{\boldsymbol{z}}, \ldots,
\boldsymbol{\psi}_{\boldsymbol{x}_{10}} \odot \boldsymbol{\psi}_{\boldsymbol{z}},
$$
see Figure \ref{fig:shift_2d_examples}.

Informally speaking, this image consists of translated versions of thresholded versions of the original. The images are overlapping, which captures some autocorrelation of the image with itself when inner products are computed. We note that classification performance seems higher when parameters are tuned so that there is overlap. The number of images in the HDC encoding is controlled by the bandwidth parameter $\lambda$.  To make another visualization, we can look at the class averages of these images; see Figure \ref{classaverage}.
\begin{figure}[ht!]
\centering
\includegraphics[width=1\textwidth, trim={0cm 1cm 0cm 1cm}]{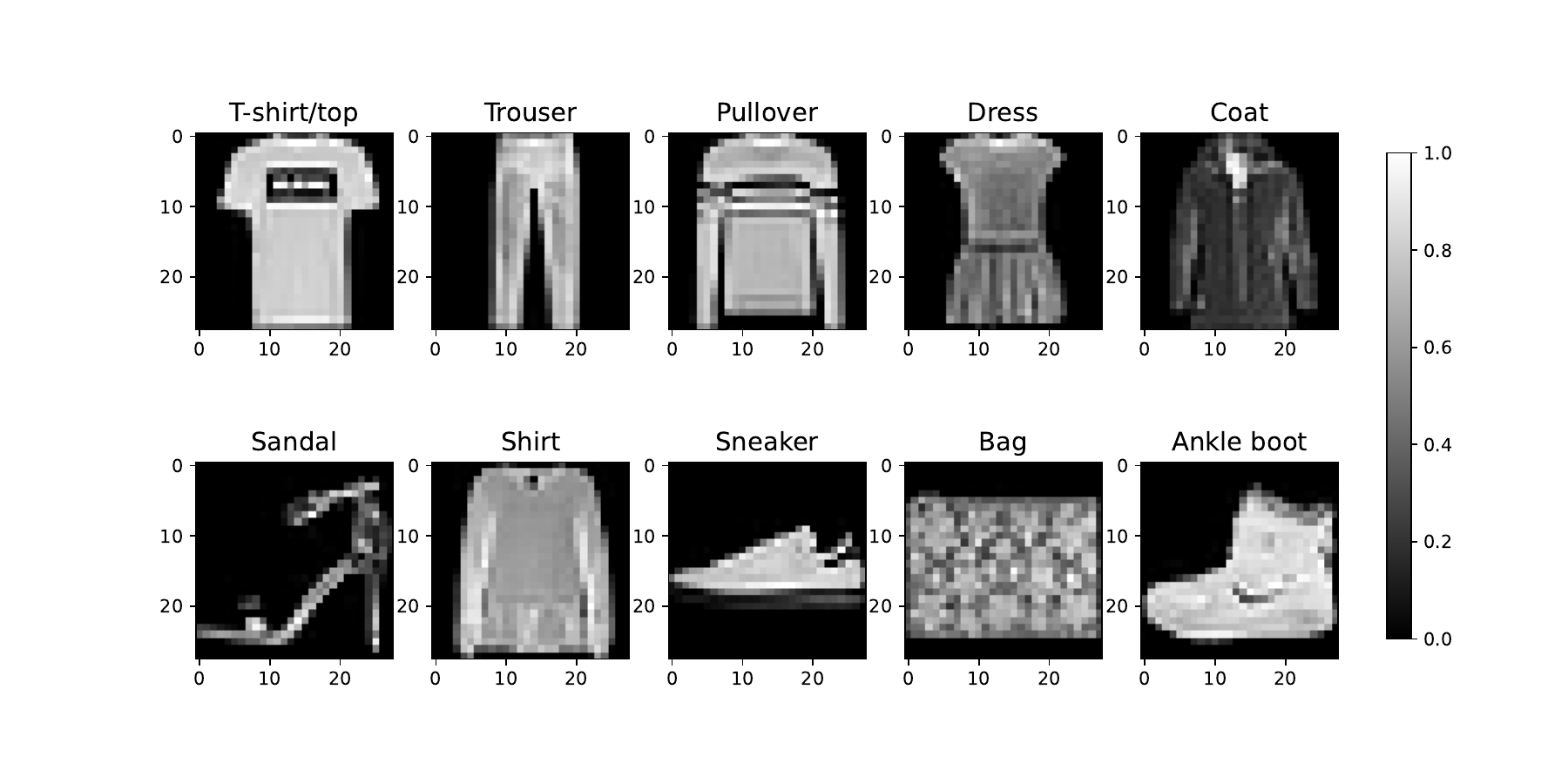}
\caption{One example from each class in Fashion MNSIT dataset}
\label{fig:fashionexamples}
\end{figure}

\begin{figure}[ht!]
\centering
\includegraphics[width=1\textwidth, trim={1cm 1cm 1cm 1cm}]{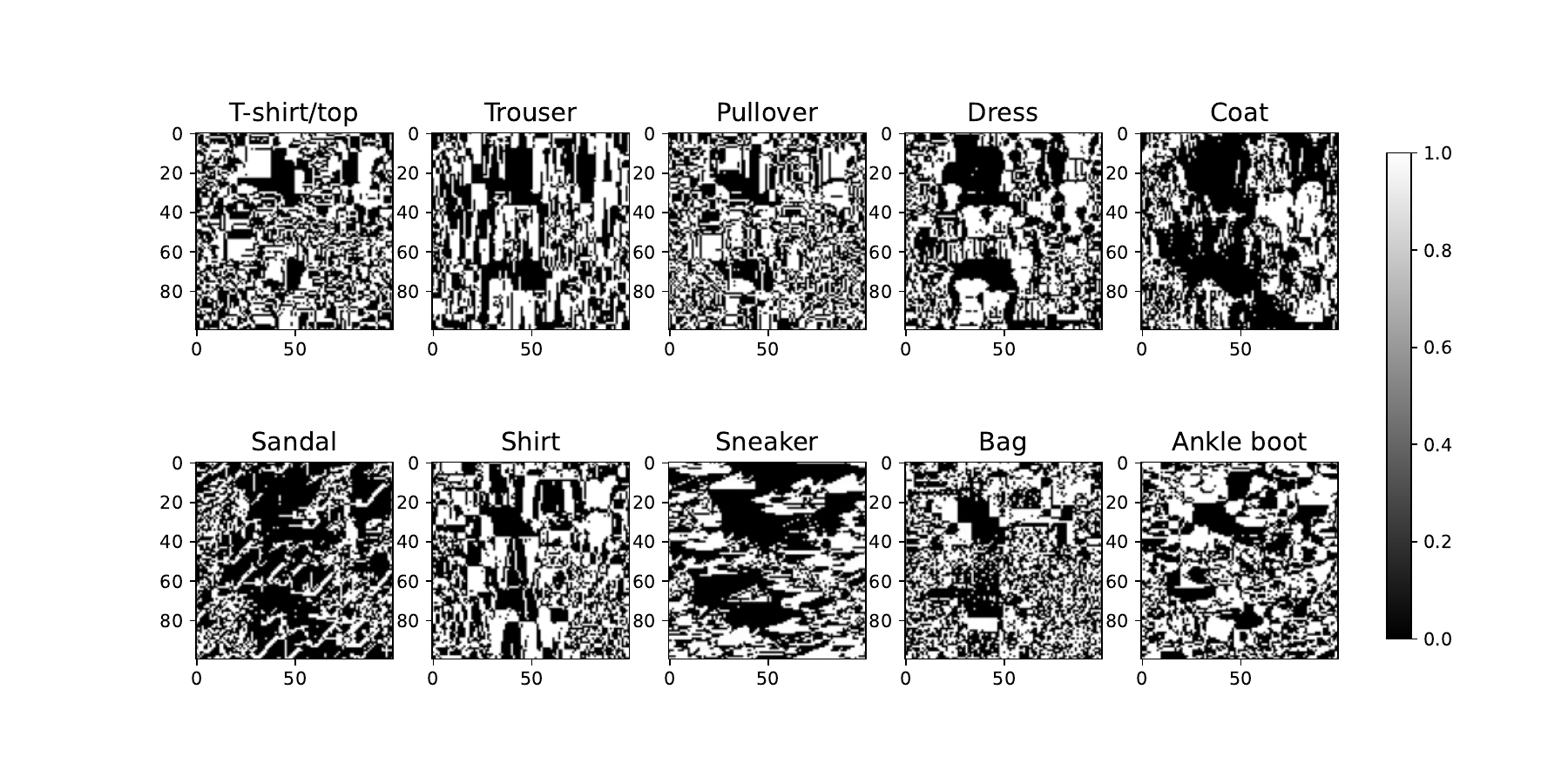}
\caption{We plot $\boldsymbol{\psi}_{\boldsymbol{x}_j} \odot \boldsymbol{\psi}_{\boldsymbol{z}}$ 
for the class example images $\boldsymbol{x}_1,\ldots,\boldsymbol{x}_{10}$  from Figure \ref{fig:fashionexamples}. }
\label{fig:shift_2d_examples}
\end{figure} 

\begin{figure}[ht!]
\centering
\includegraphics[width=1\textwidth, trim={1cm 1cm 1cm 1cm}]{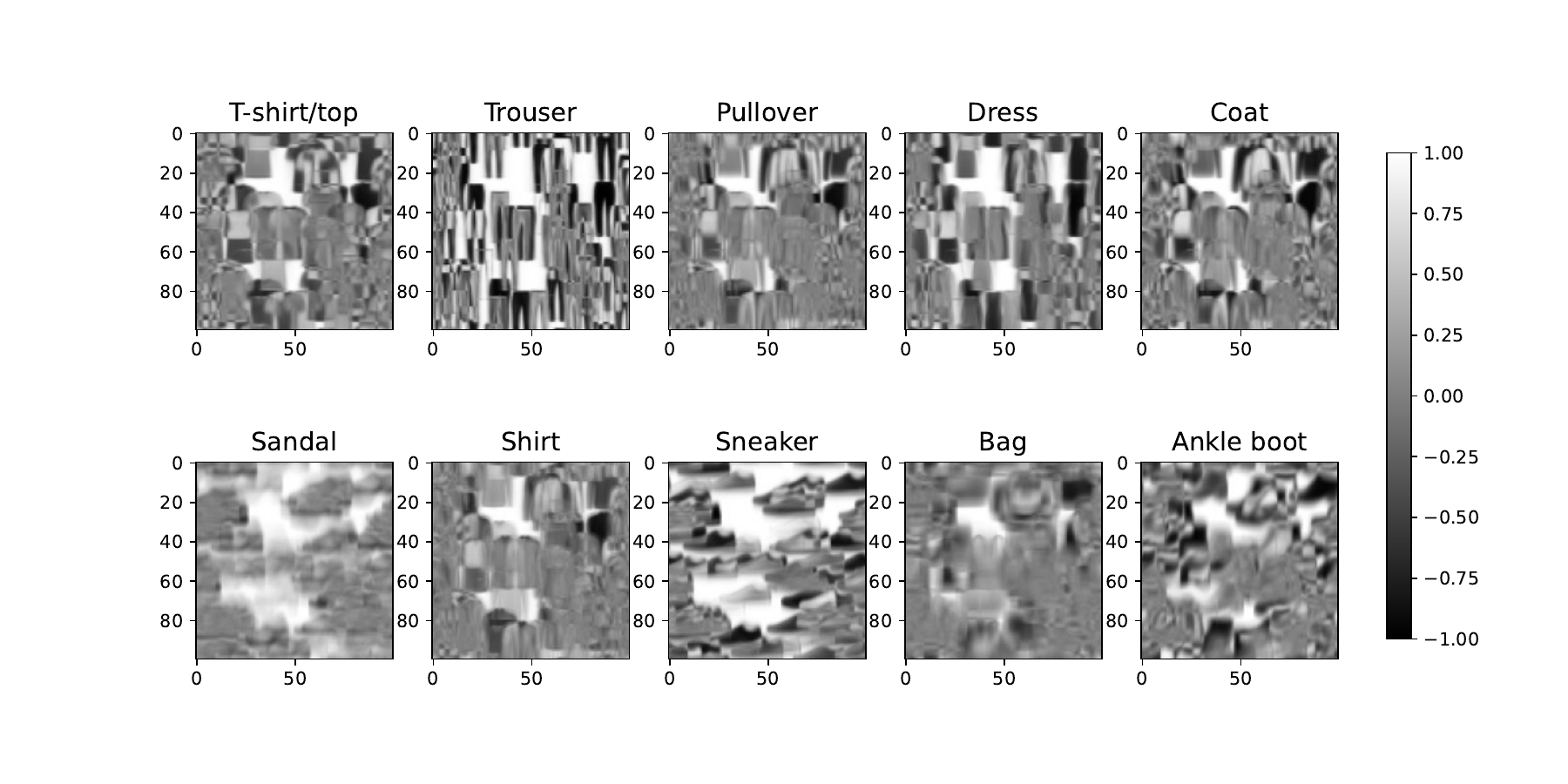}
\caption{We average $\frac{1}{\# (C_j)} \sum_{\boldsymbol{x} \in C_j} \boldsymbol{\psi}_{\boldsymbol{x}} \odot \boldsymbol{\psi}_{\boldsymbol{z}}$ where $C_j$ is the $j$-th class, for $j=1,\ldots,10$. }
\label{classaverage}
\end{figure}

\section{Discussion} \label{discussion}

This paper introduces Laplace-HDC, a binary HDC scheme motivated by the geometry of hypervector data encodings.  We build upon the work of Yu \emph{et al.} \cite{yu2022understanding}, by considering the inner product structure of hypervector encodings resulting from the binding operation for hypervectors with a covariance structure.

We show that the Laplace kernel (rather than the Gaussian kernel) is a natural choice for the covariance structure of the hypervectors used in this construction. 
 In this case, we show that the inner product of the hypervector encodings of data is related to an $\ell_1$-norm Laplace kernel between data points, which motivates a method for setting the bandwidth parameter $\lambda > 0$ in this construction. These observations lead to a practical binary HDC scheme, which we call Laplace-HDC.

Our results also indicate a limitation of binary HDC schemes of this type for image data: the spatial relationship between pixels is lost.
 More precisely, our results show that the inner product structure of hypervector encodings is invariant to global permutations of the data. We demonstrate that when spatial relationships are encoded, even in an elementary way, such as through convolutional Haar features (which are not invariant to global permutations of the pixels), the accuracy of binary HDC improves for image data.

We note that more complicated feature extraction methods could be used to increase the performance further. For example, the features derived from the output of one or more layers of a convolutional neural network trained on image data could be used. The fact that binary HDC schemes of this type are invariant to global permutations of the pixels can be viewed both as a limitation or a feature of binding-based binary HDC encoding schemes.

We emphasize that our theoretical results only say that spatial relationships are not encoded in the inner product structure of the binding operation. It may be possible to recover spatial information via another method. We illustrate such a method when we define a translation-equivariant binary HDC encoding scheme for images, which is a potential direction for future work.

We note that the Trace-Orthongal assumption that we make on the families of permutation matrices we consider may be overly restrictive. We performed some limited experiments using families of permutations, which are each sampled independently and uniformly, which achieved similar accuracy to the Trace-Orthogonal families of permutation we considered (1D-Cyclic, 1D-Block, 2D-Cyclic, 2D-Block). However, it should be noted that each of the families of permutations we considered has efficient implementations that maintain memory locality when encoding batched images. In contrast, performing a uniformly random permutation is orders of magnitude slower due to a lack of memory locality. However, there may be pseudo-random permutations that can be efficiently implemented that are interesting to consider.

\subsection*{Acknowledgements} 
The authors thank Peter Cowal for useful discussions about the paper.

\bibliographystyle{siam}
\bibliography{references}

\begin{appendix}
\section{Proof of Analytic Results}

\subsection{Proof of Theorem \ref{thm1}} \label{proofthm1}

\begin{proof}[Proof of Theorem~\ref{thm1}]
By the definition \eqref{generalencoding:scheme} of the map $\boldsymbol{x} \mapsto \embed_{\boldsymbol{x}}$ we have
 \begin{align} \notag 
S(\boldsymbol{x}, \boldsymbol{y}) 
&= \frac{1}{N}\mathbb{E} \Tr\left(\embed_{\boldsymbol{x}}\embed_{\boldsymbol{y}}^\top \right)
\\&=\notag 
\frac{1}{N} \mathbb{E} \Tr \left( \left( \bigodot_{i=1}^d \boldsymbol{\Pi}_i \boldsymbol{v}_{\boldsymbol{x}(i)}
\right) \left( \bigodot_{i=1}^d  \boldsymbol{v}_{\boldsymbol{y}(i)}^\top \boldsymbol{\Pi}_i^\top 
\right) \right) 
\\&=\notag 
\frac{1}{N} \mathbb{E} \Tr \left( \bigodot_{i=1}^d \boldsymbol{\Pi}_i \boldsymbol{v}_{\boldsymbol{x}(i)}
\boldsymbol{v}_{\boldsymbol{y}(i)}^\top \boldsymbol{\Pi}_i^\top 
\right)   
\\&=\notag 
\frac{1}{N}  \sum_{j=1}^N \mathbb{E} \left( \prod_{i=1}^d \boldsymbol{e}_j^\top \boldsymbol{\Pi}_i \boldsymbol{v}_{\boldsymbol{x}(i)}
\boldsymbol{v}_{\boldsymbol{y}(i)}^\top \boldsymbol{\Pi}_i^\top \boldsymbol{e}_j
\right)  
\\&=\notag
\frac{1}{N}  \sum_{j=1}^N  \prod_{i=1}^d \mathbb{E}\left(\boldsymbol{e}_j^\top \boldsymbol{\Pi}_i \boldsymbol{v}_{\boldsymbol{x}(i)}
\boldsymbol{v}_{\boldsymbol{y}(i)}^\top \boldsymbol{\Pi}_i^\top \boldsymbol{e}_j
\right)
\\\notag &
=
\frac{1}{N}  \sum_{j=1}^N  \prod_{i=1}^d \boldsymbol{K}(\boldsymbol{x}(i),\boldsymbol{y}(i)) 
\\&  \label{eqSxy}
= \prod_{i=1}^d \boldsymbol{K}(\boldsymbol{x}(i),\boldsymbol{y}(i))  .
\end{align}
In the chain of equalities above, the third equality holds since for arbitrary vectors $\boldsymbol{u}_i, \boldsymbol{v}_i,\mathbb{R}^N$, 
\[
\bigodot_{i=1}^d \boldsymbol{u}_i \bigodot_{i'=1}^d \boldsymbol{v}_{i'}^\top = \bigodot_{i=1}^d \boldsymbol{u}_i\boldsymbol{v}_i^\top,
\]
where the right-side $\bigodot$ represents a straightforward generalization of Hadamard product from vectors to matrices. Moreover, by the construction of $\boldsymbol{v}_k$, the elements $\boldsymbol{v}_k(j)$ and $\boldsymbol{v}_{k'}(j')$ are statistically independent whenever $j\neq j'$. As a result, since the permutation matrices $\boldsymbol{\Pi}_i$ are non-overlapping, the factors $\boldsymbol{e}_j^\top \boldsymbol{\Pi}_i \boldsymbol{v}_{\boldsymbol{x}(i)} \boldsymbol{v}_{\boldsymbol{y}(i)}^\top \boldsymbol{\Pi}_i^\top \boldsymbol{e}_j$ and $\boldsymbol{e}_j^\top \boldsymbol{\Pi}_{i'} \boldsymbol{v}_{\boldsymbol{x}(i')}
\boldsymbol{v}_{\boldsymbol{y}(i')}^\top \boldsymbol{\Pi}_{i'}^\top \boldsymbol{e}_j$ become independent whenever $i\neq i'$, which justifies the fifth equality. Finally, the sixth equality is a straightforward implication of \eqref{simvivj:scalar}.

Next, we bound the variance. First, we compute the second moment
\begin{equation} \label{approxstepthm1}
\begin{split}
\mathbb{E} \left( \frac{ \embed_{\boldsymbol{x}}^\top \embed_{\boldsymbol{y}}}{N} \right)^2 
&=  \frac{1}{N^2}  \mathbb{E} \left( \sum_{j=1}^N  \prod_{i=1}^d \boldsymbol{e}_j^\top \boldsymbol{\Pi}_i \boldsymbol{v}_{\boldsymbol{x}(i)}
\boldsymbol{v}_{\boldsymbol{y}(i)}^\top \boldsymbol{\Pi}_i^\top \boldsymbol{e}_j \right)^2 \\
&= \frac{1}{N^2}  \mathbb{E} \left( \sum_{j=1}^N  \prod_{i=1}^d \boldsymbol{e}_j^\top \boldsymbol{\Pi}_i \boldsymbol{v}_{\boldsymbol{x}(i)}
\boldsymbol{v}_{\boldsymbol{y}(i)}^\top \boldsymbol{\Pi}_i^\top \boldsymbol{e}_j \right)        \left( \sum_{j'=1}^N  \prod_{i'=1}^d \boldsymbol{e}_{j'}^\top \boldsymbol{\Pi}_{i'} \boldsymbol{v}_{\boldsymbol{x}(i')}
\boldsymbol{v}_{\boldsymbol{y}(i')}^\top \boldsymbol{\Pi}_{i'}^\top \boldsymbol{e}_{j'} \right)\\
&=  \frac{1}{N^2}   \sum_{j=1}^N  \sum_{j'=1}^N 
\mathbb{E} \left(
\prod_{i=1}^d \boldsymbol{e}_j^\top \boldsymbol{\Pi}_i \boldsymbol{v}_{\boldsymbol{x}(i)}
\boldsymbol{v}_{\boldsymbol{y}(i)}^\top \boldsymbol{\Pi}_i^\top \boldsymbol{e}_j \prod_{i'=1}^d \boldsymbol{e}_{j'}^\top \boldsymbol{\Pi}_{i'} \boldsymbol{v}_{\boldsymbol{x}(i')}
\boldsymbol{v}_{\boldsymbol{y}(i')}^\top \boldsymbol{\Pi}_{i'}^\top \boldsymbol{e}_{j'}\right).
\end{split}
\end{equation}
Define the $N\times N$ matrix 
\[
\boldsymbol{Q} := \sum_{i'=1}^d\sum_{i=1}^d\boldsymbol{\Pi}_i\boldsymbol{\Pi}_{i'}^\top,
\]
and accordingly, define
\[
\Omega := \left\{(j,j'): \boldsymbol{Q}(j,j')>0\right\}, 
\]
where $\boldsymbol{Q}(j,j')$ is the $(j,j')$-th element of $\textbf{Q}$. One can split the summation above over $\Omega$ and $\Omega^c$ as:
\begin{equation}\label{EOEOc}
    \mathbb{E} \left( \frac{ \embed_{\boldsymbol{x}}^\top \embed_{\boldsymbol{y}}}{N} \right)^2 = E_\Omega + E_{\Omega^c},
\end{equation}
where for $\mathcal{X}=\Omega,\Omega^c$:
\[
E_\mathcal{X} = \frac{1}{N^2}  \mathbb{E} \sum_{(j,j')\in\mathcal{X}} \prod_{i=1}^d \boldsymbol{e}_j^\top \boldsymbol{\Pi}_i \boldsymbol{v}_{\boldsymbol{x}(i)}
\boldsymbol{v}_{\boldsymbol{y}(i)}^\top \boldsymbol{\Pi}_i^\top \boldsymbol{e}_j \prod_{i'=1}^d \boldsymbol{e}_{j'}^\top \boldsymbol{\Pi}_{i'} \boldsymbol{v}_{\boldsymbol{x}(i')}
\boldsymbol{v}_{\boldsymbol{y}(i')}^\top \boldsymbol{\Pi}_{i'}^\top \boldsymbol{e}_{j'}. 
\]
The summand is clearly upper-bounded by 1, so the summation of terms  over $\Omega$ is bounded by
\begin{equation}\label{EO}
    E_{\Omega} \leq \frac{\sum_{(j,j')\in \Omega} 1}{N^2} =
    \frac{|\Omega|}{N^2} = 
    \frac{\gamma_{\mathcal{P}}}{N^2},
\end{equation}
where the final equality follows from the definition of $\gamma_\mathcal{P}$.
On the other hand, for $(j,j')\in \Omega^c$ we have
\begin{equation}\label{EOc}
 \mathbb{E} \prod_{i=1}^d \boldsymbol{e}_j^\top \boldsymbol{\Pi}_i \boldsymbol{v}_{\boldsymbol{x}(i)}
\boldsymbol{v}_{\boldsymbol{y}(i)}^\top \boldsymbol{\Pi}_i^\top \boldsymbol{e}_j \prod_{i'=1}^d \boldsymbol{e}_{j'}^\top \boldsymbol{\Pi}_{i'} \boldsymbol{v}_{\boldsymbol{x}(i')}
\boldsymbol{v}_{\boldsymbol{y}(i')}^\top \boldsymbol{\Pi}_{i'}^\top \boldsymbol{e}_{j'}
= S(\boldsymbol{x},\boldsymbol{y})^2. 
\end{equation}
To see why \eqref{EOc} holds, we start by showing that if  $(j,j')\in\Omega^c$, then 
\begin{equation} \label{eqpermnoteq}
\boldsymbol{e}_{j}^\top \boldsymbol{\Pi}_{i} \not = \boldsymbol{e}_{j'}^\top \boldsymbol{\Pi}_{i'}, \quad \forall i,i' \in \{1,\ldots,d\}.
\end{equation}
Suppose not, that is, suppose $\boldsymbol{e}_{j}^\top \boldsymbol{\Pi}_{i_0} = \boldsymbol{e}_{j'}^\top \boldsymbol{\Pi}_{i'_0}$, for some $i_0,i_0' \in \{1,\ldots,d\}$. Then,
$$
0 < \boldsymbol{e}_{j}^\top \boldsymbol{\Pi}_{i_0}  \boldsymbol{\Pi}_{i'_0}^\top 
\boldsymbol{e}_{j'}
\le \boldsymbol{e}_{j}^\top\sum_i\sum_{i'} \boldsymbol{\Pi}_{i}  \boldsymbol{\Pi}_{i'}^\top \boldsymbol{e}_{j'}=\boldsymbol{Q}(j,j'),
$$
which contradicts the fact that $(j,j') \in \Omega^c$. Recall that, by construction, the elements $\boldsymbol{v}_k(j)$ and $\boldsymbol{v}_{k'}(j')$ are statistically independent whenever $j\neq j'$.
By \eqref{eqpermnoteq}, it follows that $\boldsymbol{e}_j^\top \boldsymbol{\Pi}_i \boldsymbol{v}_{\boldsymbol{x}(i)}
\boldsymbol{v}_{\boldsymbol{y}(i)}^\top \boldsymbol{\Pi}_i^\top \boldsymbol{e}_j$ and $\boldsymbol{e}_{j'}^\top \boldsymbol{\Pi}_{i'} \boldsymbol{v}_{\boldsymbol{x}(i')}
\boldsymbol{v}_{\boldsymbol{y}(i')}^\top \boldsymbol{\Pi}_{i'}^\top \boldsymbol{e}_{j'}$ 
are independent, for all $i,i' \in \{1,\ldots,d\}$. Thus, we can take the expectation of the product of $i$ and $i'$ separately and use the fact that
\[\prod_{i=1}^d \mathbb{E}\left(\boldsymbol{e}_j^\top \boldsymbol{\Pi}_i \boldsymbol{v}_{\boldsymbol{x}(i)}
\boldsymbol{v}_{\boldsymbol{y}(i)}^\top \boldsymbol{\Pi}_i^\top \boldsymbol{e}_j
\right) = S(\textbf{x},\textbf{y}),
\]
to deduce \eqref{EOc}. Combining \eqref{EOEOc}, \eqref{EO}, and \eqref{EOc} gives
$$
\mathbb{E} \left( \frac{ \embed_{\boldsymbol{x}}^\top \embed_{\boldsymbol{y}}}{N} \right)^2 
\le  \frac{\gamma_\mathcal{P}}{N^2}  + \frac{1}{N^2} \sum_{(j,j')\in\Omega^c} S(\boldsymbol{x},\boldsymbol{y})^2 = \frac{\gamma_{\mathcal{P}}}{N^2} + \frac{N^2 - \gamma_\mathcal{P}}{N^2} S(\boldsymbol{x},\boldsymbol{y})^2.
$$
Finally, we have
\begin{equation}
\begin{split}
\Var \left( \frac{ \embed_{\boldsymbol{x}}^\top \embed_{\boldsymbol{y}}}{N} \right) &= 
\mathbb{E} \left( \frac{ \embed_{\boldsymbol{x}}^\top \embed_{\boldsymbol{y}}}{N} \right)^2  - \left( \mathbb{E}
 \frac{ \embed_{\boldsymbol{x}}^\top \embed_{\boldsymbol{y}}}{N} \right)^2 \\
 & \le \frac{\gamma_\mathcal{P}}{N^2} + \frac{N^2 -  \gamma_\mathcal{P}}{N^2} S(\boldsymbol{x},\boldsymbol{y})^2 - S(\boldsymbol{x}, \boldsymbol{y})^2 \\ 
 & = \frac{\gamma_\mathcal{P}}{N^2} (1-S(\boldsymbol{x},\boldsymbol{y})^2) \\ 
 & \le \frac{2\gamma_\mathcal{P}}{N^2} (1-S(\boldsymbol{x},\boldsymbol{y})),\\ 
\end{split}
\end{equation}
where the final inequality follows from using the fact that $1+S(\boldsymbol{x},\boldsymbol{y}) \le 2$.
This completes the proof.

\end{proof}

\begin{remark}[Sharpness of proof of Theorem \ref{thm1}]
From  \eqref{approxstepthm1} we have
$$
\mathbb{E} \left( \frac{ \embed_{\boldsymbol{x}}^\top \embed_{\boldsymbol{y}}}{N} \right)^2 
=  \frac{1}{N^2}   \sum_{j=1}^N  \sum_{j'=1}^N 
\mathbb{E} \left(\prod_{i=1}^d X_{j,i} \prod_{i'=1}^d X_{j',i'}\right) ,
$$
where 
$$
X_{j,i} = \boldsymbol{e}_j^\top \boldsymbol{\Pi}_i \boldsymbol{v}_{\boldsymbol{x}(i)}
\boldsymbol{v}_{\boldsymbol{y}(i)}^\top \boldsymbol{\Pi}_i^\top \boldsymbol{e}_j .
$$
In the following, we show this estimate can be refined leading to a more complicated statement.
Let 
$$
\Omega_{j,j'} = \left\{ i \in \{1,\ldots,d\}: \boldsymbol{e}_j^\top \boldsymbol{\Pi}_i \sum_{i'=1}^d \boldsymbol{\Pi}^\top_{i'} \boldsymbol{e}_{j'} > 0 \right\}.
$$
Using this notation, we have
\begin{equation} \label{omegajjprime}
\mathbb{E} \left( \prod_{i=1}^d X_{j,i} \prod_{i'=1}^d X_{j',i'} \right)  = 
\prod_{i \in \Omega_{j,j'}^c} \mathbb{E} \left(X_{j,i}\right)  \prod_{i' \in \Omega_{j',j}^c}  \mathbb{E} \left(X_{j',i'}\right)  \mathbb{E} \left( \prod_{i \in \Omega_{j,j'}} X_{j,i} \prod_{i' \in \Omega_{j',j}} X_{j',i'} \right),
\end{equation}
where we used the fact that $X_{j,i}$ is independent of all other random variables in the product when $i \in \Omega_{j,j'}^c$, and likewise, by symmetry, $X_{j,i'}$ is independent of all other random variables in the product when $i' \in \Omega_{j',j}$. We have
$$
 \prod_{i \in \Omega_{j,j'}^c} \mathbb{E} X_{j,i}  \prod_{i' \in \Omega_{j',j}^c}  \mathbb{E} X_{j',i'} =
\prod_{i \in \Omega_{j,j'}^c} \boldsymbol{K}(\boldsymbol{x}(i),\boldsymbol{y}(i))  \prod_{i' \in \Omega_{j',j}^c}  \boldsymbol{K}(\boldsymbol{x}(i'),\boldsymbol{y}(i')), 
$$
while the third term on the right-hand side of \eqref{omegajjprime} is bounded by $1$. 
In fact, this bound is sharp in certain cases, for example, in the case when $\boldsymbol{x}$ and $\boldsymbol{y}$ are constant. In summary, we have
$$
\mathbb{E} \left( \frac{ \embed_{\boldsymbol{x}}^\top \embed_{\boldsymbol{y}}}{N} \right)^2 
\le  \frac{1}{N^2}   \sum_{j=1}^N  \sum_{j'=1}^N 
\prod_{i \in \Omega_{j,j'}^c} \boldsymbol{K}(\boldsymbol{x}(i),\boldsymbol{y}(i))  \prod_{i' \in \Omega_{j',j}^c}  \boldsymbol{K}(\boldsymbol{x}(i'),\boldsymbol{y}(i')).
$$
Note that with this notation $\Omega_{j,j'}$ is the set of ``bad permutations'',
which cause terms in the product of random variables to be dependent. When $\Omega_{j,j'} = \emptyset$ and $\Omega_{j',j} = \emptyset$ we have
$$
\prod_{i \in \Omega_{j,j'}^c} \boldsymbol{K}(\boldsymbol{x}(i),\boldsymbol{y}(i))  \prod_{i' \in \Omega_{j',j}^c}  \boldsymbol{K}(\boldsymbol{x}(i'),\boldsymbol{y}(i')) = S(\boldsymbol{x},\boldsymbol{y})^2.
$$
Alternatively, if there is only one bad permutation in each set $\Omega_{j,j'} = \{i_0\}$ and $\Omega_{j',j} = \{i_0'\}$ we have
$$
\prod_{i \in \Omega_{j,j'}^c} \boldsymbol{K}(\boldsymbol{x}(i),\boldsymbol{y}(i))  \prod_{i' \in \Omega_{j',j}^c}  \boldsymbol{K}(\boldsymbol{x}(i'),\boldsymbol{y}(i')) 
= \frac{S(\boldsymbol{x},\boldsymbol{y})^2}{\boldsymbol{K}(\boldsymbol{x}(i_0),\boldsymbol{y}(i_0))
\boldsymbol{K}(\boldsymbol{x}(i_0'),\boldsymbol{y}(i_0'))}.
$$
Previously, we performed a global analysis, which only looked at the number of $j,j'$ for which there are any bad permutations (see $\Omega$ below). This refined analysis shows that the number of bad permutations also matters. The block-based permutation schemes minimize the number of $j,j'$ for which $\Omega_{j,j'} > 0$, but do this bad trying to maximize $|\Omega_{j,j'}|$ whenever $\Omega_{j,j'}>0$.

\end{remark}

\subsection{Proof of Theorem \ref{admissiblefamily}} \label{proofthm2}

\begin{proof}[Proof of Theorem \ref{admissiblefamily}]
If $\boldsymbol{K}_\alpha$ is defined by \eqref{admissiblekerneleq}, then
$$
\boldsymbol{W}_\alpha(i,j) = \exp \left( - \frac{\pi^2}{8} \lambda |i - j|^{2 \alpha} \right),
$$
which is positive semi-definite by Lemma \ref{lemschoenberg} and the fact that $\alpha \in (0,1]$.
Note that 
$$
\frac{2}{\pi} \arcsin \left( \exp \left( -x^2 \right) \right) = 1 - \frac{2 \sqrt{2} x}{\pi} + \mathcal{O}(x^3),
$$
as $x \to 0$. If $\lambda |i -j|^\alpha \le \varepsilon$, then it follows that the kernel $\boldsymbol{K}_\alpha$ defined in \eqref{admissiblekerneleq} satisfies
\begin{equation} \label{eq:kexpand}
\boldsymbol{K}_\alpha(i,j) = 1 - \lambda|i -j|^\alpha + \mathcal{O}(\varepsilon^3).
\end{equation}
Write
$$
S_\alpha(\boldsymbol{x},\boldsymbol{y}) =     \prod_{i=1}^d 
\boldsymbol{K}_\alpha(\boldsymbol{x}(i),\boldsymbol{y}(j))
= \exp\left( \sum_{i=1}^d \ln \left( \boldsymbol{K}_\alpha(\boldsymbol{x}(i),\boldsymbol{y}(j)) \right) \right).
$$
Using the series expansions 
$$
\ln(1 - x) = -x + \mathcal O(x^2), \quad \text{and} \quad \exp(x) = 1 + \mathcal O(x),
$$
as $x \rightarrow 0$, together with \eqref{eq:kexpand} gives
$$
\exp\left( \sum_{i=1}^d \ln \left( \boldsymbol{K}_\alpha(\boldsymbol{x}(i),\boldsymbol{y}(j)) \right) \right) 
=
\exp\left( - \lambda \sum_{i=1}^d |\boldsymbol{x}(i) - \boldsymbol{y}(i)|^\alpha \right)\left(1   + \mathcal{O}(\varepsilon^2 d) \right).
$$
That is,
$$
S_\alpha(\boldsymbol{x},\boldsymbol{y}) = 
\exp\left( - \lambda \| \boldsymbol{x} - \boldsymbol{y}\|^\alpha_\alpha \right)\left(1   + \mathcal{O}(\varepsilon^2 d) \right)
$$
which completes the proof.

\end{proof}

\end{appendix}

\end{document}